\documentclass{article}

\PassOptionsToPackage{numbers, compress}{natbib}

\usepackage[final]{neurips_2022}

\usepackage[utf8]{inputenc} 
\usepackage[T1]{fontenc}    
\usepackage[pagebackref,breaklinks,colorlinks,citecolor=green,linkcolor=red,urlcolor=blue]{hyperref}       
\usepackage{url}            
\usepackage{booktabs}       
\usepackage{amsfonts}       
\usepackage{nicefrac}       
\usepackage{microtype}      
\usepackage{xcolor}         
\usepackage[pdftex]{graphicx}

\usepackage{float}
\usepackage{caption}
\usepackage{subcaption}
\usepackage{xspace}
\usepackage{bm}
\usepackage{amsmath}
\usepackage{amssymb}
\usepackage{textcomp}
\usepackage{mathtools}
\usepackage{multirow}
\usepackage{array}
\usepackage{wrapfig}
\usepackage{amsthm}
\usepackage{nth}
\usepackage{dsfont}
\usepackage{algorithm}
\usepackage{algpseudocode}
\usepackage{stackengine}
\usepackage{mathrsfs}
\usepackage{tabularx}

\usepackage{listings}
\usepackage[frozencache]{minted}
\usepackage{enumitem}

\setminted[python]{ %
    linenos=true,             
    autogobble=true,          
    frame=lines,
    framesep=2mm,
    fontsize=\footnotesize,
    escapeinside=||
}

\usepackage[capitalize]{cleveref}
\crefname{section}{Sec.}{Secs.}
\Crefname{section}{Section}{Sections}
\Crefname{table}{Table}{Tables}
\crefname{table}{Tab.}{Tabs.}
\crefname{algorithm}{Alg.}{Algs.}

\definecolor{green}{HTML}{00aa00}
\definecolor{blue}{HTML}{0000aa}
\definecolor{red}{HTML}{aa0000}

\usepackage[autostyle]{csquotes}
\theoremstyle{definition}
\newtheorem*{proposition}{Proposition}

\newtheorem*{definition}{Definition}
\newtheorem*{corollary}{Corollary}

\newtheorem{assumption}{Asm.}

\theoremstyle{remark}
\newtheorem*{remark}{Remark}

\def\app#1#2{%
  \mathrel{%
    \setbox0=\hbox{$#1\sim$}%
    \setbox2=\hbox{%
      \rlap{\hbox{$#1\propto$}}%
      \lower1.1\ht0\box0%
    }%
    \raise0.25\ht2\box2%
  }%
}

\newcommand{\diffH}[1]{\textcolor{green}{\rlap{$_{\uparrow #1}$}}}

\newcommand\barbelow[1]{\stackunder[1.2pt]{$#1$}{\rule{3ex}{.075ex}}}
\newcommand\barbelowshort[1]{\stackunder[1.2pt]{$#1$}{\rule{1ex}{.075ex}}}

\newcommand\barabove[1]{\stackon[1.2pt]{$#1$}{\rule{1.7ex}{.075ex}}}
\newcommand\baraboveshort[1]{\stackon[1.2pt]{$#1$}{\rule{1ex}{.075ex}}}

\DeclareMathOperator*{\argmin}{arg\,min}

\DeclarePairedDelimiter\parentheses{\lparen}{\rparen}

\DeclareMathOperator*{\MVBer}{MVB}

\DeclarePairedDelimiterX\set[1]\lbrace\rbrace{#1}

\newcommand{\MVB}[2]{\MVB \parentheses*{#1,\,#2}}

\newcommand{\smallsim}{\smallsym{\mathrel}{\sim}}

\newcommand*{\estimates}{\mathrel{\widehat=}}

\makeatletter
\DeclareRobustCommand\onedot{\futurelet\@let@token\@onedot}
\def\@onedot{\ifx\@let@token.\else.\null\fi\xspace}

\def\eg{\emph{e.g}\onedot} 
\def\ie{\emph{i.e}\onedot} 
 
\def\etc{\emph{etc}\onedot} 
\def\wrt{\emph{w.r.t}\onedot}

\newcommand{\smallsym}[2]{#1{\mathpalette\make@small@sym{#2}}}
\newcommand{\make@small@sym}[2]{%
  \vcenter{\hbox{$\m@th\downgrade@style#1#2$}}%
}
\newcommand{\downgrade@style}[1]{%
  \ifx#1\displaystyle\scriptstyle\else
    \ifx#1\textstyle\scriptstyle\else
      \scriptscriptstyle
  \fi\fi
}

\newcommand{\subalign}[1]{%
  \vcenter{%
    \Let@ \restore@math@cr \default@tag
    \baselineskip\fontdimen10 \scriptfont\tw@
    \advance\baselineskip\fontdimen12 \scriptfont\tw@
    \lineskip\thr@@\fontdimen8 \scriptfont\thr@@
    \lineskiplimit\lineskip
    \ialign{\hfil$\m@th\scriptstyle##$&$\m@th\scriptstyle{}##$\hfil\crcr
      #1\crcr
    }%
  }%
}

\newcommand{\eat}[1]{}

\newenvironment{code}
 {\RecustomVerbatimEnvironment{Verbatim}{BVerbatim}{}%
  \def\FV@BProcessLine##1{%
    \hbox{%
      \hbox to\z@{\hss\theFancyVerbLine\kern\FV@NumberSep}%
      \FancyVerbFormatLine{##1}%
    }%
  }%
  \VerbatimEnvironment
  \setbox\z@=\hbox\bgroup
  \begin{minted}{python}}
  {\end{minted}\egroup
  \leavevmode\vbox{\hrule\kern2mm\box\z@\kern2mm\hrule}}

\makeatother

\title{A Lower Bound of Hash Codes' Performance}
\author{%
Xiaosu Zhu\textsuperscript{1}\\
\texttt{xiaosu.zhu@outlook.com}
\And
Jingkuan Song\textsuperscript{1}\thanks{Corresponding author.}\\
\texttt{jingkuan.song@gmail.com}
\And
Yu Lei\textsuperscript{1}\\
\texttt{leiyu6969@gmail.com}
\And
Lianli Gao\textsuperscript{1}\\
\texttt{lianli.gao@uestc.edu.cn}
\And
Heng Tao Shen\textsuperscript{1,2}\\
\texttt{shenhengtao@hotmail.com}
\And
\mdseries \textsuperscript{1}Center for Future Media, University of Electronic Science and Technology of China \\
\textsuperscript{2}Peng Cheng Laboratory
}

\begin{document}

\maketitle

\begin{abstract}

    As a crucial approach for compact representation learning, hashing has achieved great success in effectiveness and efficiency. Numerous heuristic Hamming space metric learning objectives are designed to obtain high-quality hash codes. Nevertheless, a theoretical analysis of criteria for learning good hash codes remains largely unexploited. In this paper, we prove that inter-class distinctiveness and intra-class compactness among hash codes determine the lower bound of hash codes' performance. Promoting these two characteristics could lift the bound and improve hash learning. We then propose a surrogate model to fully exploit the above objective by estimating the posterior of hash codes and controlling it, which results in a low-bias optimization. Extensive experiments reveal the effectiveness of the proposed method. By testing on a series of hash-models, we obtain performance improvements among all of them, with an up to $26.5\%$ increase in mean Average Precision and an up to $20.5\%$ increase in accuracy. Our code is publicly available at \url{https://github.com/VL-Group/LBHash}.
\end{abstract}

\section{Introduction}
The explosive increase of multimedia digital content requires people to develop large-scale processing techniques for effective and efficient multimedia understanding~\cite{GLDv2,Survey1,zeng2021conceptual,zeng2022video}. Meanwhile, growing focus on carbon neutrality initiatives urges researchers to handle above issues with low power and memory consumption~\cite{HashSurvey3,XOR}. Fortunately, hashing, as a key role in compact representation learning, has demonstrated efficiency and effectiveness over the past few decades towards above demands~\cite{CompactLearning1}. By digesting multimedia contents into short binary codes, hashing is able to significantly reduce storage, boost speed, and preserve key information~\cite{DPgQ,Hash3,Hash2,Hash1,McQuic}.

In recent years, learning to hash has shown its advantages in alleviating information loss by performing data-dependent Hamming space projection with a non-linear model~\cite{Hash4,HammingMetric,SpectralHash} compared to hand-crafted hashing functions. The purpose of optimizing hash-models is to generate semantically meaningful hash codes for downstream tasks. A practical convention is to increase similarities among codes of the same category, otherwise the opposite. Such convention is validated to be effective in many applications~\cite{application1,DSDH,MultiIdxHash,GreedHash,application2,application3}.

Practices in current advances still leave two questions which remain primarily unexploited: When should we determine a hash-model for producing \textit{good} codes, and by what means to achieve this goal? Taking a comprehensive study on the above questions is essential since it not only guides us to develop high-performance algorithms in typical hash-based applications but also provides potential possibilities to inspire future works in hash learning.

Although some works~\cite{E2LSH,SpectralHash} try to answer the above questions by building connections between hash codes and common descriptors with domain knowledge, we still seek theoretical guarantees for hash codes' performance. Similar situations occur in related works~\cite{HashSurvey3,HashSurvey2,HashSurvey1}, where objectives to optimize hash-models are designed by intuitions or heuristics.

By formulating correlations among hash codes as inter-class distinctiveness and intra-class compactness, we try to address the above issues and provide a lower bound of hash codes' performance. Based on this, we could further derive from them as objectives to lift the lower bound. Moreover, to effectively train model towards such targets, estimating and further controlling the posterior of hash codes is necessary since bit-level dependence exists in posterior distribution. If we neglect this, bias will be introduced during hash-model optimization, hindering performance. Summarizing both of them, Our main contributions are summarized as follows:



1) To the best of our knowledge, we are the first to formulate a lower bound of hash codes' performance, which is directly proportional to inter-class distinctiveness and intra-class compactness among codes. Based on this, an objective is further proposed to lift this lower bound and thus enhance performance.

2) A novel posterior estimation over hash codes is proposed to perform low-bias optimization towards the above objective. As a non-linear function, it could be utilized for optimizing hash-models via gradient descent. Such plug-and-play component is able to integrate into current hash-models to boost performance.

3) Extensive experiments on three benchmark datasets reveal the effectiveness of the proposed method in two typical tasks. Specifically, we obtain an up to $26.5\%$ increase on mean Average Precision for fast retrieval. We also observe an up to $20.5\%$ increase on accuracy for hash-based recognition.

The rest of paper is organized as follows: We first briefly review current hashing methods and give preliminaries (\cref{Sec.RelatedWork,Sec.Preliminaries}). Then, lower bound on hash codes' performance by inter-class distinctiveness and intra-class compactness is proposed (\cref{Sec.DisPar}). To fully exploit such guidance, posterior estimation over multivariate Bernoulli distribution is designed for code control (\cref{Sec.DisMod}). Experiments (\cref{Sec.Exp}) are then conducted for evaluation.

\section{Related Works}
\label{Sec.RelatedWork}
This paper will focus on data-dependent learning-to-hash methods. Recent advances in building these models are validated to be effective in hash-based recognition and fast retrieval~\cite{E2LSH,OneLoss1,PQ,DSDH,GreedHash}. The typical approach utilizes a projection function to convert raw inputs into compact binary codes and conduct classification or approximate nearest neighbour search on codes for downstream tasks~\cite{application1,application2,application3}. As mentioned in introduction, to obtain such a projector, current practices suggest that we formulate it as a Hamming space metric learning problem under binary constraints. Pointwise~\cite{Hash3,DiscreteGraphHash,Pointwise1}, pairwise~\cite{DCH,HashNet,Hash2,DSDH,DRQ,DBDH}, triplet~\cite{triplet2,triplet1}, listwise~\cite{listhash,listwise1} and prototype-based~\cite{NetVLAD,AnchorHash,OneLoss1,CSQ} objectives are all studied for model training, where the core idea is to adjust similarities among samples from same / different categories.

Besides the above objectives, we should also note that hash-model optimization is NP-hard~\cite{E2LSH,DiscreteGraphHash,GreedHash,SpectralHash}. That is caused by the discrete, binarized hash code formulation. To tackle the issue, previous works propose many methods that target on following purposes: \textbf{a)} Surrogate functions that approximate hashing by continuous functions, such as $\mathit{tanh}$~\cite{HashNet}, Stochastic neuron~\cite{StochasticHash}, $\mathit{affine}$~\cite{ubgan} or Straight-Through~\cite{STE}. By adopting these, the discrete hash-model optimization is transformed into approximated but easier one via gradient descent. \textbf{b)} Non-differentiable optimization algorithm design, which focuses on solutions that directly handle discrete inputs, including Coordinate descent~\cite{SDH}, Discrete gradient estimation~\cite{GradientEst}, \etc. There is also bag of tricks which help for stable training. For example, minimizing quantization error reduces the gap between raw outputs of the model and final hash codes~\cite{DSDH}, and code-balancing alleviates the problem of producing trivial codes~\cite{OneLoss2}.

Please refer to literature reviews on above works~\cite{HashSurvey3,HashSurvey2,HashSurvey1} for comprehensive studies on learning to hash. We could see two issues exist and remain primarily unexploited in above works. The first is objective design, which is commonly proposed by intuitive and with few theoretical guarantees. The second is model optimization, where bit-level dependence is essential in code control but most works neglect it. To tackle these, we would firstly give a lower bound of hash codes' performance and utilize it as a guidance for hash-model training. Then, a novel model optimization approach is proposed in a multivariate perspective for low-bias code control.

\section{Preliminaries}
\label{Sec.Preliminaries}

Learning to hash performs data-dependent compact representation learning. Inputs $\bm{x} \in \bm{\mathcal{X}} \subseteq \bm{\mathbb{R}}^d$ are firstly transformed into $h$-dim real-valued vector $\bm{\ell}$ by $\mathcal{F}_{\bm{\theta}}$, a projection parameterized by ${\bm{\theta}}$: $\bm{\mathbb{R}}^d \xrightarrow{\mathcal{F}_{\bm{\theta}}} \bm{\mathbb{R}}^{h}$. And $\bm{\ell}$ is then binarized to $h$ bits binary code $\bm{b} \in \bm{\mathcal{B}} \subseteq \bm{\mathbb{H}}^h$:
\begin{eqnarray*}
    \bm{x} \xrightarrow{\mathcal{F}_{\bm{\theta}}} \bm{\ell} \xrightarrow{\operatorname{bin}} \bm{b},\; \mathit{where} \; \operatorname{bin}\left(\cdot\right) = \left\{
    \begin{aligned}
    +1, \;   \left(\cdot\right) \geq 0, \\
    -1, \;   \left(\cdot\right) < 0.
    \end{aligned}\right.
\end{eqnarray*}

\textbf{Typical usage of hash codes.} Common hashing-based tasks include fast retrieval~\cite{NetVLAD,isoHahs,Hash1,3DUVQ} and recognition~\cite{DSDH,HammingMetric,SDH,GreedHash}. For fast retrieval, given a query under the same distribution of $\bm{\mathcal{X}}$, we convert it into query hash code $\bm{q}$ and conduct fast approximate nearest neighbour search in $\bm{\mathcal{B}}$ to produce rank list $\mathcal{R}$. In $\mathcal{R}$, samples are organized from nearest to furthest to query. Correspondingly among all results in $\mathcal{R}$, true positives $\bm{\mathit{tp}}^{i} \in \bm{\mathit{TP}}$ indicate samples that match with query while others are false positives $\bm{\mathit{fp}}^{i} \in \bm{\mathit{FP}}$ ($i$ indicates rank). As a criterion, Average Precision (AP) is commonly adopted to determine how well the retrieved results are. It encourages true positives to have higher ranks than false positives: $\mathit{AP} = \frac{1}{\lvert\bm{\mathit{TP}}\rvert} \sum_{\forall \bm{\mathit{tp}}^{i,m} \in \bm{\mathit{TP}}}{P@i}$, where $P@i$ is the rank-$i$ precision. As for hash-based recognition, since current works formulate it as a variant of retrieval, we could still adopt AP as a criterion to indicate performance. The detailed explanation is placed in Supp. Sec. C.




\section{Lower Bound by Inter-Class Distinctiveness and Intra-Class Compactness}
\label{Sec.DisPar}
We start at an arbitrary rank list for studying the characteristics of codes in the above scenarios.
To calculate AP easily, we first introduce \textit{mis-rank} that indicates how many false positives take higher places than true positives. We refer readers to Supp. Secs. A and B for detailed description and proof.
\begin{definition}
    Mis-rank $m$ indicates how many false positives have higher ranks than a true positive, \eg, $\bm{tp}^{i, m}$ is at rank $i$, meanwhile $m \!=\! \lvert\left\{ d\left(\bm{q}, \bm{\mathit{fp}} \right) < d\left(\bm{q}, \bm{tp}^{i, m} \right) \right\}\rvert$, where $d\left(\bm{q}, \bm{\cdot}\right)$ is distance between $\bm{q}$ and a sample, and $\lvert \, \cdot \, \rvert$ is number of elements in set.
\end{definition}
\begin{remark}
    The rank-$i$ precision $P@i$ is derived to be $\nicefrac{\left(i - m\right)}{i}$.
\end{remark}

Then, average precision could be derived as:
\begin{equation*}
    \mathit{AP} = \frac{1}{\lvert\bm{\mathit{TP}}\rvert} \sum_{\forall \bm{\mathit{tp}}^{i,m} \in \bm{\mathit{TP}}}{P@i} = \frac{1}{\vert \bm{\mathit{TP}} \rvert}\sum_{\forall \bm{\mathit{tp}}^{i,m} \in \bm{\mathit{TP}}}{\frac{i - m}{i}}, \;
\end{equation*}
for all true positives in rank list. This is because the rank-$i$ precision $P@i$ for any true positive $\bm{tp}^{i, m}$ equals to $\nicefrac{\left(i - m\right)}{i}$.

From the derivation, we could immediately obtain that $\mathit{AP}$ increases $\mathit{iff} \; m$ decreases. Therefore, now we could focus on how to reduce $m$ to thereby increase average precision. Noticed that the value of $m$ is highly related to two distances, $d\left(\bm{q}, \bm{\mathit{fp}}\right)$ and $d\left(\bm{q}, \bm{\mathit{tp}}\right)$, the following proposition is raised:
\begin{proposition}
    \begin{equation*}
        \barabove{m} \propto \frac{\max{d\left(\bm{q}, \bm{\mathit{tp}}\right)}}{\min{d\left(\bm{q}, \bm{\mathit{fp}}\right)}} \; \forall \bm{\mathit{tp}} \in \bm{\mathit{TP}}, \; \bm{\mathit{fp}} \in \bm{\mathit{FP}}
    \end{equation*}
    where $\baraboveshort{\cdot}$ denotes upper bound. Correspondingly,
    \begin{equation*}
        \barbelow{\mathit{AP}} \propto \frac{\min{d\left(\bm{q}, \bm{\mathit{fp}}\right)}}{\max{d\left(\bm{q}, \bm{\mathit{tp}}\right)}} \; \forall \bm{\mathit{tp}} \in \bm{\mathit{TP}}, \; \bm{\mathit{fp}} \in \bm{\mathit{FP}}
    \end{equation*}
    where $\barbelowshort{\cdot}$ denotes lower bound.
\end{proposition}

If inputs $\bm{x}^j$ have associated labels $y^j$, the above studies are able to be further generalized. Class-wise centers, inter-class distinctiveness and intra-class compactness are introduced to formulate our final lower-bound of hash codes' performance.
\begin{definition}
    Center $\bm{c}^c \in \bm{C}$ is a representative hash code of a specific class $c$:
\begin{equation*}
    \bm{c}^c = \argmin_{\bm{b}}{\sum_{j}{d\left(\bm{b}, \bm{b}^j\right)}, \; \forall \left\{\bm{b}^j \mid y^{j} = c \right\} }
\end{equation*}
where $\bm{b}^j$ is the $j$-th sample in set $\bm{\mathcal{B}}$ and $y^{j}$ is the label of $\bm{b}^j$.
\end{definition}
\begin{definition}
    \textit{Inter-class distinctiveness} is $\min{\mathcal{D}_{\mathit{inter}}}$ where $\mathcal{D}_{\mathit{inter}}$ is a set that measures distances between all centers over $\bm{C}$: $\left\{d\left(\bm{c}^j, \bm{c}^k\right) \;\mid\; \forall \bm{c}^j,\bm{c}^k \in \bm{C} \right\}$. In contrast, \textit{intra-class compactness} is $\nicefrac{1}{\max{\mathcal{D}_{\mathit{intra}}}}$, where $\mathcal{D}_{\mathit{intra}} = \left\{d\left(\bm{b}^j, \bm{c}^c\right) \;\mid\; \forall y^{j} = c \right\}$ measures the distances among samples of the same class to their corresponding center.
\end{definition}
Combining above propositions and definitions, we could reach:
\begin{proposition}
\begin{equation}
\label{Eq.Hero}
    \barbelow{\mathit{AP}} \propto \frac{\min{d\left(\bm{q}, \bm{\mathit{fp}}\right)}}{\max{d\left(\bm{q}, \bm{\mathit{tp}}\right)}} \geq \mathit{const} \cdot \frac{\min{\mathcal{D}_{\mathit{inter}}}}{\max{\mathcal{D}_{\mathit{intra}}}}.
\end{equation}
\end{proposition}

This proposition reveals how AP is affected by the above two inter- and intra- factors. Actually, it could cover common scenarios as a criterion of hash codes' performance, \eg, fast retrieval and recognition, which will be explained in supplementary materials. Intuitively, lifting this lower bound would enhance performance, formulating the following objective:
\begin{align}
    \label{Eq.MaxMin}
    \mathit{maximize} &\min{\mathcal{D}_{\mathit{inter}}}, \\
    \label{Eq.MinMax}
    \mathit{minimize} &\max{\mathcal{D}_{\mathit{intra}}}, \\
    \mathit{s.t.}\; \bm{\mathcal{X}} \subseteq \bm{\mathbb{R}}^d &\xrightarrow{\mathcal{F}_{\bm{\theta}},\;\operatorname{bin}} \bm{\mathcal{B}} \subseteq \bm{\mathbb{H}}^h.  \nonumber
\end{align}
Inspired by \cite{CSQ}, we give a specific solution to implement \cref{Eq.MaxMin,Eq.MinMax} under supervised hashing by firstly defining distance maximized class-specific centers and then shrinking hash codes to their corresponding centers. We will validate the effectiveness of such a solution in experiments. It is worth noting that achieving above goal is an open topic, including unsupervised or other complex scenarios, which we leave for future study. Next, to optimize hash-models with the above supervision, we introduce posterior estimation over hash codes and then control it.



\section{Posterior Estimation for Code Control}
\label{Sec.DisMod}
It is essentially hard to minimize the distance between a hash code and a target to achieve \cref{Eq.MinMax} for two reasons. Firstly, the posterior distribution of $\bm{\mathcal{B}}$ given inputs $\bm{\mathcal{X}}$ is formulated to be under multivariate Bernoulli distribution: $p\left(\bm{\mathcal{B}} \mid \bm{\mathcal{X}} \right) \smallsim \MVBer$. If we do not take correlations among different variables into account, optimization on $\bm{\mathcal{B}}$ is biased. Secondly, dealing with hash codes is a $\left\{-1, +1\right\}$ discrete optimization, which is not trivial to handle. Therefore, we try to organize \cref{Eq.MinMax} as a Maximum Likelihood Estimation (MLE) to tackle the above issues. Considering the definition of Hamming distance between an arbitrary hash code $\bm{b}$ and its target $\bm{t}$:
\begin{equation*}
    d\left(\bm{b}, \bm{t}\right) = \sum_i^h{\mathds{1}\left\{ \bm{b}_i \neq \bm{t}_i \right\}} = \sum_i^h{\mathds{1}\left\{ \bm{\ell}_i \bm{t}_i < 0 \right\}}
\end{equation*}
where $\mathds{1}$ is the characteristic function and $\left(\bm{\cdot}\right)_i$ indicates value on the $i$-th dimension. Since Hamming distance measures how many bits are different between two codes, the probability of $d\left(\bm{b}, \bm{t}\right) = \delta$ can be formulated as:
\begin{equation*}
    p\left(d\left(\bm{b}, \bm{t}\right) = \delta\right) =  \sum_{\forall i \in {h \choose \delta},\; j \notin {h \choose \delta}}\left\{p\left( \bm{b}_i \neq \bm{t}_i, \; \bm{b}_j = \bm{t}_j\right)\right\}.
\end{equation*}
Therefore, $\bm{b}$ and $\bm{t}$ have distance $\delta \; \mathit{iff} \; p\left(d\left(\bm{b}, \bm{t}\right) = \delta\right)$ is maximized. However, to precisely calculate it is difficult, since it involves the joint probability of $\bm{b}$. Therefore, we try to estimate all joint probabilities by adopting a surrogate model $\mathcal{P}_{\bm{\pi}}$ parameterized by $\bm{\pi}$ to perform estimation:
\begin{equation}
    \label{Eq.Projection}
    \bm{\ell} \xrightarrow{\mathcal{P}_{\bm{\pi}}} \bm{o},\; \mathit{where} \; \bm{o} \in \bm{O} \subseteq \mathbb{R}^{2^h}
\end{equation}
where $\bm{o}$ is the probabilities of a Categorical distribution $p\left(\bm{O} \mid \bm{\mathcal{X}}\right) \estimates p\left(\bm{\mathcal{B}} \mid \bm{\mathcal{X}}\right)$, which contains $2^h$ entries. Each entry of $\bm{o}$ estimates one of a specific joint probability by feeding $\bm{\ell}$ into $\mathcal{P}_{\bm{\pi}}$:
\begin{equation*}
    \bm{o}_i \estimates p\left( \left(\bm{b}_k > 0\right)^{\mathds{1}\left\{\bm{\iota}_k = 1\right\}}, \left(\bm{b}_{k} < 0\right)^{\mathds{1}\left\{\bm{\iota}_k = 0\right\}} \right), \; {}^{1 \leq k \leq h,}_{0 \leq i \leq 2^h -1}
\end{equation*}
where $\bm{\iota}$ is the $h$ bits binary representation of $i$. For example, $\bm{o}_3 = p\left(\bm{b} = \left(\text{-1-1+1+1}\right)\right)$ when $h = 4$. To train $\mathcal{P}_{\bm{\pi}}$, we perform Maximum Likelihood Estimation (MLE) over $p\left(\bm{O} \mid \bm{\mathcal{X}}\right)$:
\begin{equation}
\label{Eq.SurrogateLoss}
\begin{split}
    L_{\bm{\pi}}\left(\mathcal{P}_{\bm{\pi}};\bm{o}\right) &= -\log{ \bm{o}_i }, \\
    \mathit{where} \; i &= \sum_{k=1}^{h}{\mathds{1}\left\{\bm{b}_k > 0 \right\} \cdot 2^{k - 1}}.
\end{split}
\end{equation}

If $\mathcal{P}_{\bm{\pi}}$ is ready for estimation, we could directly adopt it as a nice non-linear function to maximize $p\left(d\left(\bm{b}, \bm{t}\right) = k\right)$.
Note that to realize \cref{Eq.MinMax}, we want $\bm{b}$ and its corresponding center $\bm{c}$ are closest, \ie maximize $p\left(d\left(\bm{b}, \bm{c}\right) = 0\right)$. So, we could reformulate this maximization as another MLE:
\begin{equation}
\begin{split}
    L_{\bm{\theta}}\left(\mathcal{F}_{\bm{\theta}}; \bm{b}\right) &= -\log{p\left(d\left(\bm{b}, \bm{c}\right) = 0\right)}, \\
    \mathit{where}\; p\left(d\left(\bm{b}, \bm{c}\right) = 0\right) &= p\left( \bm{b}_i = \bm{c}_i, 1 \leq i \leq h\right),
\end{split}
\label{Eq.MLE}
\end{equation}
which could be optimized by calculating surrogate gradients of \cref{Eq.MLE} through $\mathcal{P}_{\bm{\pi}}$:
\begin{equation}
\label{Eq.Gradient}
    \frac{\hat{\partial{L_{\bm{\theta}}}}}{\partial{{\bm{\theta}}}} = \frac{\partial\left\{-\log{\bm{o}_i}; \; i = \textstyle{\sum}_{k=1}^{h}{\mathds{1}\left\{\bm{c}_k > 0 \right\} \cdot 2^{k - 1}}\right\}}{\partial{\bm{\pi}}} \frac{\partial{\bm{\pi}}}{\partial{\bm{\ell}}} \frac{\partial{\bm{\ell}}}{\partial{{\bm{\theta}}}} \estimates \frac{\partial{L_{\bm{\theta}}}}{\partial{\bm{\theta}}}.
\end{equation}

Unfortunately, such joint probability estimation on $\MVBer$ requires $\mathcal{O}\left(2^h\right)$ complexity, which is not flexible when $h$ is large. By adopting the idea of block code~\cite{BlockCode}, the above estimation is able to perform on long hash bits by separating hash codes into a series of blocks. Specifically, any $h$ bits hash codes can be split into $u$ blocks while each block consumes $\nicefrac{h}{u}$ bits. Correspondingly, $u$ independent surrogate networks $\mathcal{P}_{\bm{\pi}_i}, 1 \leq i \leq u$ are adopted to perform the above estimation and back-propagation simultaneously. With such decomposition, the whole problem is transformed into $u$ sub-problems with a reduced complexity $\mathcal{O}\left(2^h\right) \rightarrow \mathcal{O}\left(u \cdot 2^{\nicefrac{h}{u}}\right)$.

As a summarization all of things, overall optimization via gradient descent is placed in \cref{Alg.Main}. We first perform \cref{Eq.MaxMin} step by using the pre-defined centers and then perform \cref{Eq.MinMax} step by back-propagation via \cref{Eq.Gradient}. Two models $\mathcal{F}_{\bm{\theta}}, \mathcal{P}_{\bm{\pi}}$ are optimized with learning rate $\eta_1, \eta_2$ respectively.

\begin{algorithm}
\caption{One of implementations under supervised circumstance.}
\label{Alg.Main}

\begin{algorithmic}[1]
\Procedure{Train}{$\mathcal{F}_{\bm{\theta}},\; \mathcal{P}_{\bm{\pi}}$}\Comment{Training procedure of two models $\mathcal{F}_{\bm{\theta}},\;\mathcal{P}_{\bm{\pi}}$.}
\State Generate class-specific centers $\bm{c} \in \bm{C}, \lvert \bm{C} \rvert = \text{Class-num}$; \Comment{(\cref{Eq.MaxMin}).}
\Repeat\Comment{Main training loop.}
\State Sample $\bm{x}$ from $\bm{\mathcal{X}}$ with label $y$;
\State $\bm{\ell} = \mathcal{F}_{\bm{\theta}}\left(\bm{x}\right)$;
\State $\bm{o} = \mathcal{P}_{\bm{\pi}}\left(\bm{\ell}\right)$;
\State $\bm{\pi} \leftarrow \bm{\pi} - \eta_1 \frac{\partial{L_{\bm{\pi}}}}{\partial{\bm{\pi}}}$; \Comment{(\cref{Eq.SurrogateLoss}).}
\State ${\bm{\theta}} \leftarrow {\bm{\theta}} - \eta_2 \frac{\hat{\partial{L_{\bm{\theta}}}}}{\partial{{\bm{\theta}}}}$ with corresponding center $\bm{c}$; \Comment{\cref{Eq.MinMax}, \cref{Eq.Gradient}.}
\label{Alg.No8}
\Until{Total epoch exceeds;}
\State \textbf{return} $\mathcal{F}_{\bm{\theta}}$; \Comment{Optimized hash-model $\mathcal{F}_{\bm{\theta}}$}
\EndProcedure
\end{algorithmic}
\end{algorithm}

\section{Experiments}
\label{Sec.Exp}
We conduct extensive experiments on three benchmark datasets to confirm the effectiveness of our proposed method. To make fair comparisons, we first provide experiments setup.

\subsection{Setup}
Our experiments focus on performance comparisons on two typical hash-based tasks, fast retrieval and recognition, with and without integrating our proposed method. The general evaluation pipeline is to train models in the training split, then hash all samples in the base split. Then for retrieval, any hashed queries in query split are used to get rank lists of the base split from nearest to farthest according to Hamming distance. As for recognition, we adopt a $k$NN classifier or a linear model to produce queries' predictions.

\begin{wraptable}{R}{0.5\textwidth}
\caption{Total training time with different variants, where we substitute our surrogate model with BCE loss from \cite{CSQ} or Cauchy loss from \cite{DCH}.}
\label{Tab.Time}
\centering
\begin{tabular}{@{}rccc@{}}
\toprule
\multirow{2}{*}{Method} & \multicolumn{3}{c}{Training time per epoch (s)} \\ \cmidrule(lr){2-4}
                        & $16$ bits    & $32$ bits    & $64$ bits     \\ \midrule
BCE                     & $38.13$  & $38.20$  & $38.52$  \\
Cauchy                  & $38.21$  & $38.30$  & $38.42$  \\
Ours                    & $38.95$  & $39.78$  & $40.84$  \\
\bottomrule
\end{tabular}
\end{wraptable}

\subsubsection{Datasets}
Our experiments are conducted on three datasets, varying in scale, variety and perplexity to validate methods' performance in different scenarios. Both single-label and multi-label datasets are included. We follow previous works to generate three splits, which are detailed below:

\textbf{CIFAR-10~\cite{CIFAR}} is a single-label $10$-class dataset. The whole dataset contains $6,000$ images for each class. Following~\cite{HashNet}, we split the dataset into $500\; \mid \;5,400\; \mid \;100$ for each class randomly as train, base and query splits, respectively.

\textbf{NUS-WIDE~\cite{NUSWIDE}} consists of $81$ labels and images may have one or more labels. We follow previous works~\cite{HashNet} to pick the most frequent $21$ labels and their associated images ($195,834$) for experiments. Specifically, $193,734$ images are randomly picked to form the base split while remaining $2,100$ images are adopted for queries. $10,500$ images are randomly sampled from the base split for training models.

\textbf{ImageNet~\cite{ImageNet}} is a large-scale dataset consists of $1,000$ classes. To conduct experiments, we follow~\cite{HashNet} to pick a subset of $100$ classes where all images of these classes in the training set / validation set are as base split / query split respectively ($128,503\; \mid \;4,983$ images). We then randomly sample $100$ images per class in the base split for training.

\begin{figure*}[t]
    \centering
     \begin{minipage}[b][][b]{0.48\textwidth}
         \centering
         \includegraphics[width=\linewidth]{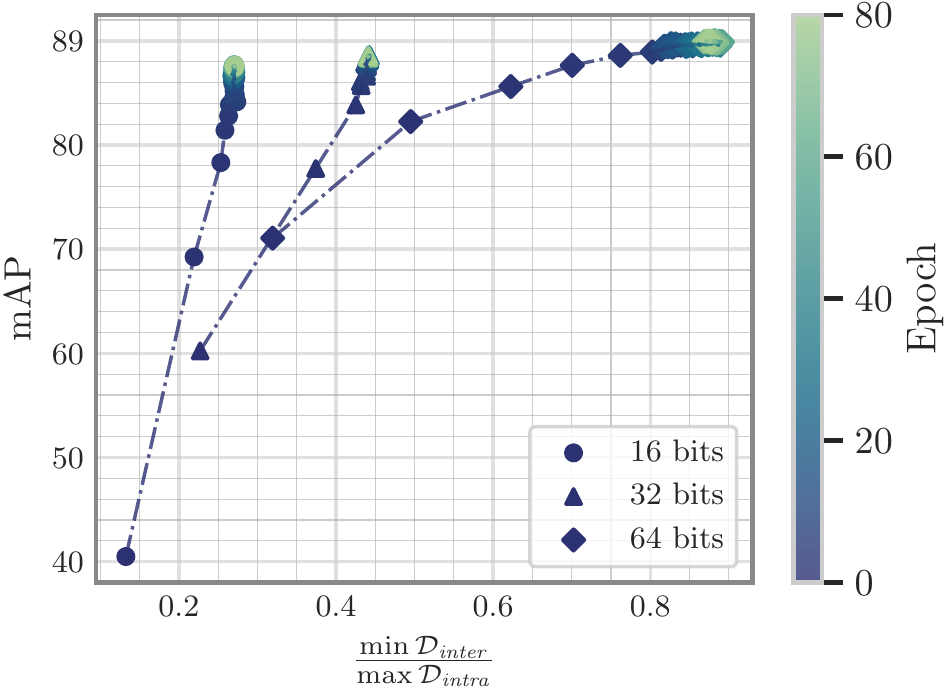}
         \caption{mAP \wrt $\frac{\min{\mathcal{D}_{\mathit{inter}}}}{\max{\mathcal{D}_{\mathit{intra}}}}$ per epoch on ImageNet $64$ bits during training. These two values have a positive relationship.}
         \label{Fig.Trace}
     \end{minipage}
     \hfill
     \begin{minipage}[b][][b]{0.48\textwidth}
         \centering
         \includegraphics[width=\linewidth]{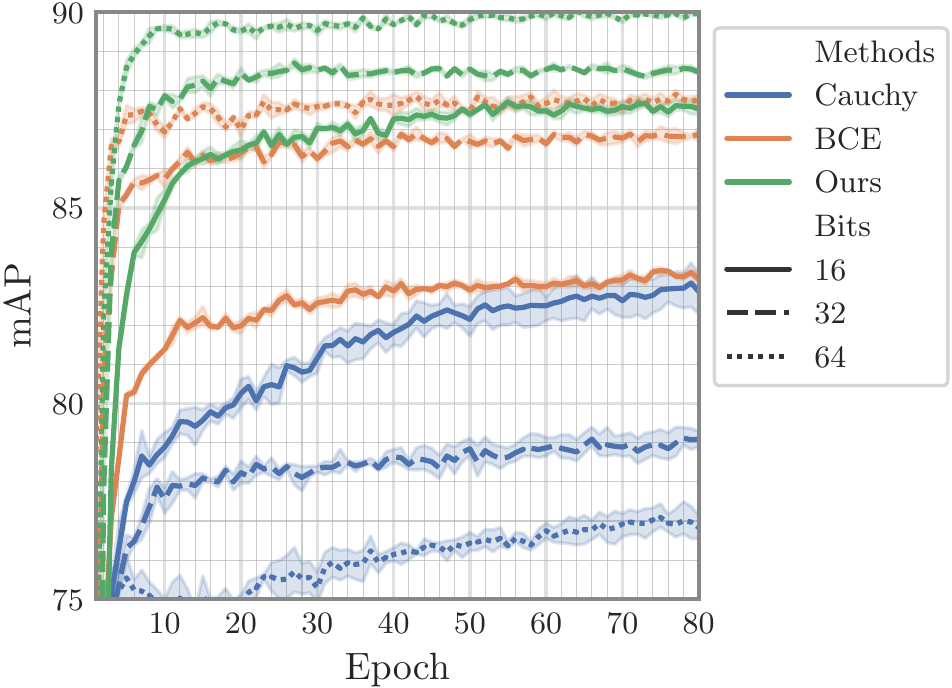}
         \caption{Convergence curve of three variants on ImageNet $64$ bits. Our $\mathcal{P}_{\bm{\pi}}$ helps for achieving higher performance than other two.}
         \label{Fig.Convergence}
     \end{minipage}
     \hfill
\end{figure*}

\subsubsection{Implementation Details}
Our method is able to integrate into common hash-models. Due to the limitation of computation resources, we choose a series of representative deep hashing methods for comparison, including \textbf{HashNet}~\cite{HashNet}, \textbf{DBDH}~\cite{DBDH}, \textbf{DSDH}~\cite{DSDH}, \textbf{DCH}~\cite{DCH}, \textbf{GreedHash}~\cite{GreedHash} and \textbf{CSQ}~\cite{CSQ}. When our method is integrated, we append a subscript $\left(\cdot\right)_D$ to the original methods' name. All methods adopted for experiments are implemented from a public benchmark with PyTorch~\cite{PyTorch}.\footnote{\url{https://github.com/swuxyj/DeepHash-pytorch}} For fair comparisons, we conduct experiments with the same backbone (ResNet-50~\cite{ResNet}) and hyper-parameters for all tested methods. We adopt Adam~\cite{Adam} optimizer with default configuration and learning rate of our method $\eta_1 = \eta_2 = 1e^{-3}$ for training. For multi-label datasets, we simply modify \cref{Alg.Main} \cref{Alg.No8} with the sum of multiple losses. Block number $u$ of $\mathcal{P}$ is set to $\nicefrac{\mathit{bits}}{8}$. For example, if the length of hash code is $64$, there will be $8$ sub-models $\mathcal{P}_{\bm{\pi}_1} \smallsim \mathcal{P}_{\bm{\pi}_8}$ trained in parallel. To mitigate randomness, we report average performance for $5$ runs on all experiments. The evaluation metric adopted for retrieval is mean Average Precision (mAP@$R$) where $R = 54,000\; \mid \;5,000\; \mid \;1,000$ on CIFAR-10, NUS-WIDE, ImageNet, respectively. For recognition performance, if methods have an auxiliary classification branch, we directly use its predictions from it. Otherwise, we adopt a $k$NN classifier built on base split and vote with $100$ nearest neighbours.

\begin{wrapfigure}{R}{0.5\linewidth}
    \centering
    \includegraphics[width=0.9\linewidth]{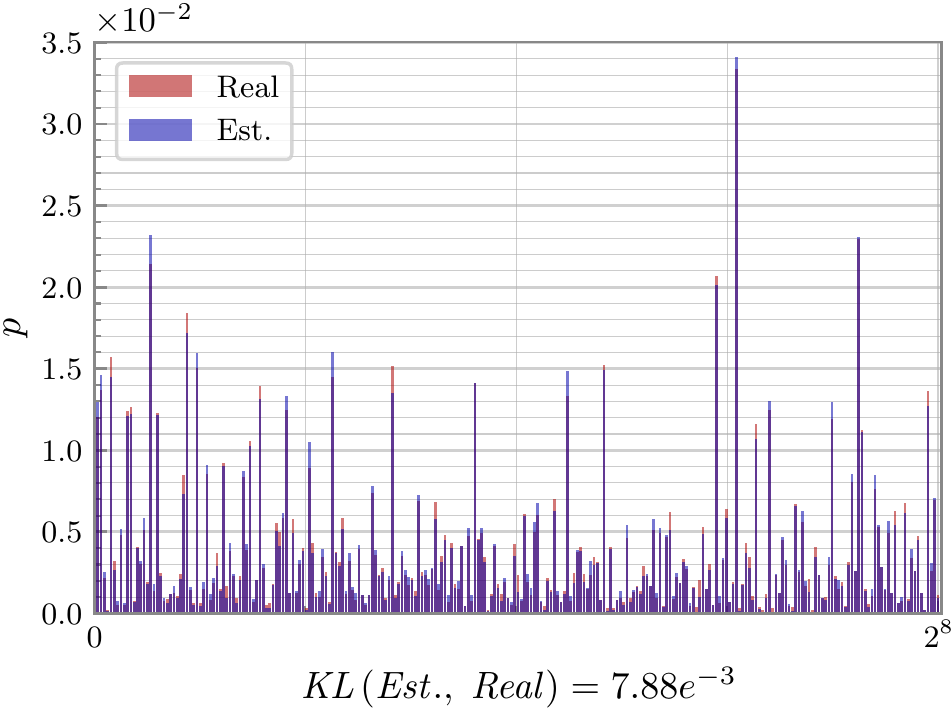}
    \newline\newline
    \begin{tabular*}{0.65\linewidth}{@{}r@{\extracolsep{\fill}}l@{}}
    \toprule
    Method               & $\mathit{KL}_{\mathit{real}}$  \\ \midrule
    Na{\"i}ve            & 0.538         \\
    Ours                 & 0.008         \\ \bottomrule
    \end{tabular*}
    \caption{Multivariate Bernoulli estimation by our proposed $\mathcal{P}_{\bm{\pi}}$, and KL-divergence comparisons with na{\"i}ve independent estimation.}
    \label{Fig.Toy}
\end{wrapfigure}

\subsection{Ablation Study}
\label{Sec.Ablation}
In the ablation study, we try to reveal thecorrectness of our proposed objectives \cref{Eq.MaxMin,Eq.MinMax}, and the effectiveness of our proposed posterior estimation model $\mathcal{P}_{\bm{\pi}}$. We answer the following questions by conducting experiments on ImageNet to demystify the above concerns. We focus on retrieval performance in the ablation study, while we observe similar results on recognition.

\textbf{Is hash-model's performance lower-bounded by \cref{Eq.Hero}?} Validating correctness of \cref{Eq.Hero} is important, but it is essentially hard to conduct experiments to confirm it. Nevertheless, we still reveal the correlation between model performance and inter-class distinctiveness / intra-class compactness by tracking mAP \wrt $\frac{\min{\mathcal{D}_{\mathit{inter}}}}{\max{\mathcal{D}_{\mathit{intra}}}}$ during training. To calculate $\mathcal{D}_{\mathit{inter}}$ and $\mathcal{D}_{\mathit{intra}}$, we first hash all samples from the base split and calculate their centers of them over all classes. The $99.9$ percentile of $\min{\mathcal{D}_{\mathit{inter}}}$ and $\max{\mathcal{D}_{\mathit{intra}}}$ is picked to avoid outliers. We conduct tracking on CSQ$_D$ with $16, 32, 64$ bits per epoch to draw \cref{Fig.Trace}. As the figure shows, lines go from lower left to upper right with small fluctuations, which indicates the linear relationship between mAP and $\frac{\min{\mathcal{D}_{\mathit{inter}}}}{\max{\mathcal{D}_{\mathit{intra}}}}$ on all bit-lengths. This may partially confirm our theoretical analysis in \cref{Sec.DisPar}.

\textbf{Convergence speed and efficiency of $\mathcal{P}_{\bm{\pi}}$.} To verify the convergence speed and efficiency of our posterior estimation model $\mathcal{P}_{\bm{\pi}}$, we test it by substituting with two variants: BCE from \cite{CSQ} and Cauchy loss from \cite{DCH}. mAP is evaluated per epoch for three variants with three bit-lengths $16, 32, 64$ and the convergence curve is plotted in \cref{Fig.Convergence}. From the figure, we could see that ours and BCE's convergence speed are similar, while Cauchy is slightly slow. Meanwhile, our method has continuously higher performance than other two after $\smallsim\!10$ epochs for all bit-lengths, which is potentially due to our method could perform low-bias optimization. Note that when training with Cauchy loss, performance is instead lower when codes are longer. This may be caused by the inner-product based Cauchy loss could not handle long bits. We also measure averaged training time per epoch of all variants on a single NVIDIA RTX 3090, which is placed in \cref{Tab.Time}. In this table, our model does not consume significant longer time than others to train. The above observations reveal the efficiency of our proposed method. It could be a drop-in replacement without introducing significant overhead. Moreover, convergence speed and efficiency will not drop when the bit-length becomes long, which reveals the flexibility of the blocked code design.

\textbf{Could $\mathcal{P}_{\bm{\pi}}$ estimate multivariate Bernoulli distribution?} We design an extra toy experiments to validate the ability of estimating multivariate Bernoulli distribution by our surrogate model $\mathcal{P}_{\bm{\pi}}$. Specifically, $8$ bits $\MVBer$ is generated with randomly $256$ joint probabilities. We then take $10,000$ samples from it as inputs to train $\mathcal{P}_{\bm{\pi}}$. The model will further estimate distribution by feeding another $100$ samples and taking the mean of all $\bm{o}$s (\cref{Eq.Projection}) as result. The estimated distribution is evaluated by Kullback–Leibler divergence~\cite{KLDiv} ($\mathit{KL}$) with real distribution, as well as bar plot, demonstrated in \cref{Fig.Toy}. As the figure shows, our predicted joint probabilities almost cover real probabilities and $\mathit{KL}$ is low. As a comparison, if we estimate real distribution with the product of edge probabilities directly ($\mathit{w/o}$ correlations between variables), $\mathit{KL}$ will be significantly increased (row $1$ in table). Such toy experiment reveals that our method is better for estimating $\MVBer$s than the na{\"i}ve one.

\begin{wrapfigure}{R}{0.5\linewidth}
    \centering
    \includegraphics[width=\linewidth]{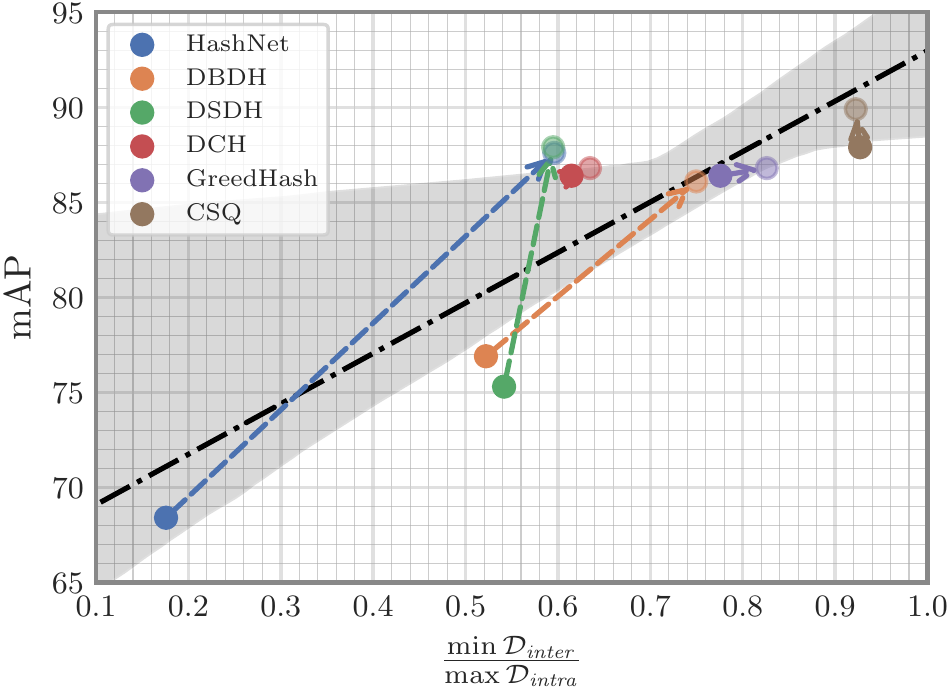}
    \caption{mAP \wrt $\frac{\min{\mathcal{D}_{\mathit{inter}}}}{\max{\mathcal{D}_{\mathit{intra}}}}$ for different methods. By integrating our auxiliary objective, we could observe significant performance and $\frac{\min{\mathcal{D}_{\mathit{inter}}}}{\max{\mathcal{D}_{\mathit{intra}}}}$ increase in most cases (from solid dots to translucent dots). Regression line is also plotted with $95\%$ confidence interval to reveal linear relationship between two metrics. }
    \label{Fig.AllMethodMAPGamma}
\end{wrapfigure}

\begin{table*}[t]
\caption{mAP comparisons on three benchmark datasets for $16, 32, 64$ bits codes. \textcolor{green}{$_{\uparrow\left(\cdot\right)}$} indicates performance enhancement with our method integrated.}
\label{Tab.mAP}
\centering
\resizebox{\linewidth}{!}{
\begin{tabular}{@{}rccccccccc@{}}
\toprule
\multirow{2}{*}{Method}&\multicolumn{3}{c}{\textbf{CIFAR-10}}&\multicolumn{3}{c}{\textbf{NUS-WIDE}}&\multicolumn{3}{c}{\textbf{ImageNet}}\\ \cmidrule(lr){2-4}\cmidrule(lr){5-7}\cmidrule(lr){8-10}
                       &  $16$ bits &  $32$ bits & $64$ bits &  $16$ bits &  $32$ bits & $64$ bits &  $16$ bits &  $32$ bits & $64$ bits \\ \midrule
HashNet                &  $63.0$    &  $81.5$   & $84.7$     &  $77.8$    &  $83.3$    & $85.0$    &  $50.6$    &  $63.1$    & $68.4$     \\
HashNet\rlap{$_D$}     & $\bm{79.1}$\diffH{16.1} & $\bm{82.3}$\diffH{0.8} & $\bm{85.3}$\diffH{0.6} & $\bm{78.6}$\diffH{0.8} & $\bm{84.0}$\diffH{0.7} & $\bm{86.2}$\diffH{1.2} & $\bm{72.6}$\diffH{22.0} & $\bm{84.3}$\diffH{21.2} & $\bm{87.6}$\diffH{19.2} \\ \midrule
DBDH                   &  $83.1$    &  $85.0$    & $85.6$    &  $82.8$    &  $84.6$    & $85.7$    &  $55.7$    &  $63.8$    & $76.9$    \\
DBDH\rlap{$_D$}        & $\bm{84.6}$\diffH{1.5} & $\bm{86.0}$\diffH{1.0} & $\bm{86.5}$\diffH{0.9} & $\bm{83.8}$\diffH{1.0} & $\bm{85.6}$\diffH{1.0} & $\bm{86.5}$\diffH{0.8} & $\bm{82.2}$\diffH{26.5} & $\bm{85.8}$\diffH{22.0} & $\bm{86.1}$\diffH{9.2}  \\ \midrule
DSDH                   &  $75.6$    &  $83.1$   & $84.5$     &  $83.3$    &  $84.5$    & $85.6$    &  $57.2$    &  $72.1$    & $75.3$        \\
DSDH\rlap{$_D$}        & $\bm{84.3}$\diffH{8.7}  & $\bm{84.6}$\diffH{1.5} & $\bm{87.3}$\diffH{2.8} & $\bm{83.6}$\diffH{0.3} & $\bm{85.3}$\diffH{0.8} & $\bm{86.4}$\diffH{0.8} & $\bm{81.7}$\diffH{24.5} & $\bm{87.3}$\diffH{15.2} & $\bm{87.9}$\diffH{12.6} \\ \midrule
DCH                    &  $83.4$    &  $84.4$   & $85.3$     &  $80.7$    &  $81.7$    & $80.9$    & $85.5$     &  $86.2$    & $86.4$      \\
DCH\rlap{$_D$}         & $\bm{83.6}$\diffH{0.2}  & $\bm{84.6}$\diffH{0.2} & $\bm{87.1}$\diffH{1.8} & $\bm{82.9}$\diffH{2.2} & $\bm{84.2}$\diffH{2.5} & $\bm{84.9}$\diffH{4.0} & $\bm{86.1}$\diffH{0.6}  & $\bm{87.5}$\diffH{1.3}  & $\bm{88.1}$\diffH{1.7}  \\ \midrule
GreedHash              &  $83.3$    &  $84.3$   & $86.9$     &  $78.6$    &  $80.3$    & $82.0$    &  $83.1$    &  $85.9$    & $86.4$        \\
GreedHash\rlap{$_D$}   & $\bm{85.2}$\diffH{1.9}  & $\bm{85.5}$\diffH{1.2} & $\bm{87.6}$\diffH{0.7} & $\bm{79.4}$\diffH{0.8} & $\bm{83.1}$\diffH{2.8} & $\bm{85.7}$\diffH{3.7} & $\bm{83.8}$\diffH{0.7}  & $\bm{86.6}$\diffH{0.7}  & $\bm{86.8}$\diffH{0.4}  \\ \midrule
CSQ                    &  $83.2$    &  $83.4$    & $84.7$    &  $82.0$    &  $83.5$    & $84.6$    &  $83.4$    &  $86.9$    & $87.9$     \\
CSQ\rlap{$_D$}         & $\bm{88.7}$\diffH{5.5} & $\bm{89.2}$\diffH{5.8} & $\bm{90.3}$\diffH{5.6} & $\bm{83.3}$\diffH{1.3} & $\bm{85.3}$\diffH{1.8} & $\bm{85.8}$\diffH{1.2} & $\bm{88.5}$\diffH{5.1} & $\bm{89.5}$\diffH{2.6} & $\bm{90.2}$\diffH{2.3}   \\ \bottomrule
\end{tabular}}
\end{table*}

\subsection{Performance Enhancements when Integrating into State-of-the-Art}
To evaluate performance gain when integrating our proposed method into hash-models, we conduct experiments on three datasets. Specifically, we first report the original performance of tested models and then re-run with our method incorporated to make a comparison. Both $16, 32, 64$ bits results are reported to show performance from short codes to long codes.

\begin{wraptable}{R}{0.5\textwidth}
\caption{Recognition performance on ImageNet.}
\label{Tab.cls}
\centering
\begin{tabular}{@{}rccc@{}}
\toprule
\multirow{2}{*}{Method}&\multicolumn{3}{c}{\textbf{ImageNet}}\\
\cmidrule(lr){2-4}
                       &  $16$ bits &  $32$ bits & $64$ bits \\ \midrule
HashNet                &  $69.1$    &  $74.1$   & $82.4$       \\
HashNet\rlap{$_D$}     & $\bm{83.2}$\diffH{14.1} & $\bm{88.1}$\diffH{14.0} & $\bm{89.5}$\diffH{7.1} \\ \midrule
DBDH                   &  $80.9$    &  $82.9$   & $88.2$   \\
DBDH\rlap{$_D$}        & $\bm{86.8}$\diffH{5.9} & $\bm{89.1}$\diffH{6.2} & $\bm{88.4}$\diffH{0.2}  \\ \midrule
DSDH                   &  $40.6$    &  $49.9$   & $54.1$       \\
DSDH\rlap{$_D$}        & $\bm{61.1}$\diffH{20.5} & $\bm{62.1}$\diffH{12.2} & $\bm{63.2}$\diffH{9.1} \\ \midrule
DCH                    &  $88.5$    &  $89.2$   & $88.1$        \\
DCH\rlap{$_D$}         & $\bm{88.9}$\diffH{0.4} & $\bm{89.5}$\diffH{0.3} & $\bm{88.7}$\diffH{0.6}  \\ \midrule
GreedHash              &  $79.1$    &  $86.8$   & $87.0$        \\
GreedHash\rlap{$_D$}   & $\bm{80.2}$\diffH{1.1} & $\bm{87.1}$\diffH{0.3} & $\bm{88.4}$\diffH{1.4}  \\ \midrule
CSQ                    &  $87.7$    &  $89.0$   & $89.4$      \\
CSQ\rlap{$_D$}         & $\bm{87.9}$\diffH{0.2} & $\bm{89.1}$\diffH{0.1} & $\bm{91.0}$\diffH{1.6}   \\
\bottomrule
\end{tabular}
\end{wraptable}

\textbf{Retrieval Performance.} Retrieval performance comparisons under mAP is shown in \cref{Tab.mAP}. With our integrated method, all tested hashing methods have performance enhancement. We observe an up to $26.5\%$ increase and $5.02\%$ on average, which is a significant improvement. Specifically, performance increase on HashNet, DBDH and DSDH is higher than other methods (first $3$ rows), especially on ImageNet dataset. A potential reason is that all of these three methods use pairwise metric learning objectives and inner-product to approximate Hamming distance, which may not handle optimization well when dataset size and class number is large. Meanwhile, our method works on multi-label datasets where all methods also obtain performance gain, indicating that multi-label data may also benefit from our objective. Furthermore, we also observe a $3.46\%$ average mAP increase on CSQ. It adopts a similar objective but a different optimization approach with ours. This increase shows the effectiveness of our proposed posterior estimation approach, which could perform better optimization than theirs.

\textbf{Recognition Performance.} Similarly, we report the recognition performance of these methods by measuring classification accuracy on ImageNet with $16, 32, 64$ bits codes. Results are reported in \cref{Tab.cls}. Generally, we observe similar performance gains. When methods have our integration, their accuracy is significantly increased, up to $20.5\%$ and $5.3\%$ on average. Specifically, some of them utilize an auxiliary classification branch to produce predictions, \ie, GreedyHash and DSDH. While for others, we adopt a $k$NN classifier to classify samples by voting with the base split. We obtain increase for both approaches. This indicates that our approach also helps for hash-based recognition in different scenarios, by leveraging inter-class distinctiveness and intra-class compactness.

From the above comparisons, we confirm the effectiveness of our proposed method. Overall, our method is validated to be flexible and promising to deploy in the above scenarios with various bits.

\begin{figure*}[t]
\captionsetup[subfigure]{justification=centering}
    \centering
    \begin{subfigure}[b]{0.245\textwidth}
     \centering
     \includegraphics[width=\linewidth]{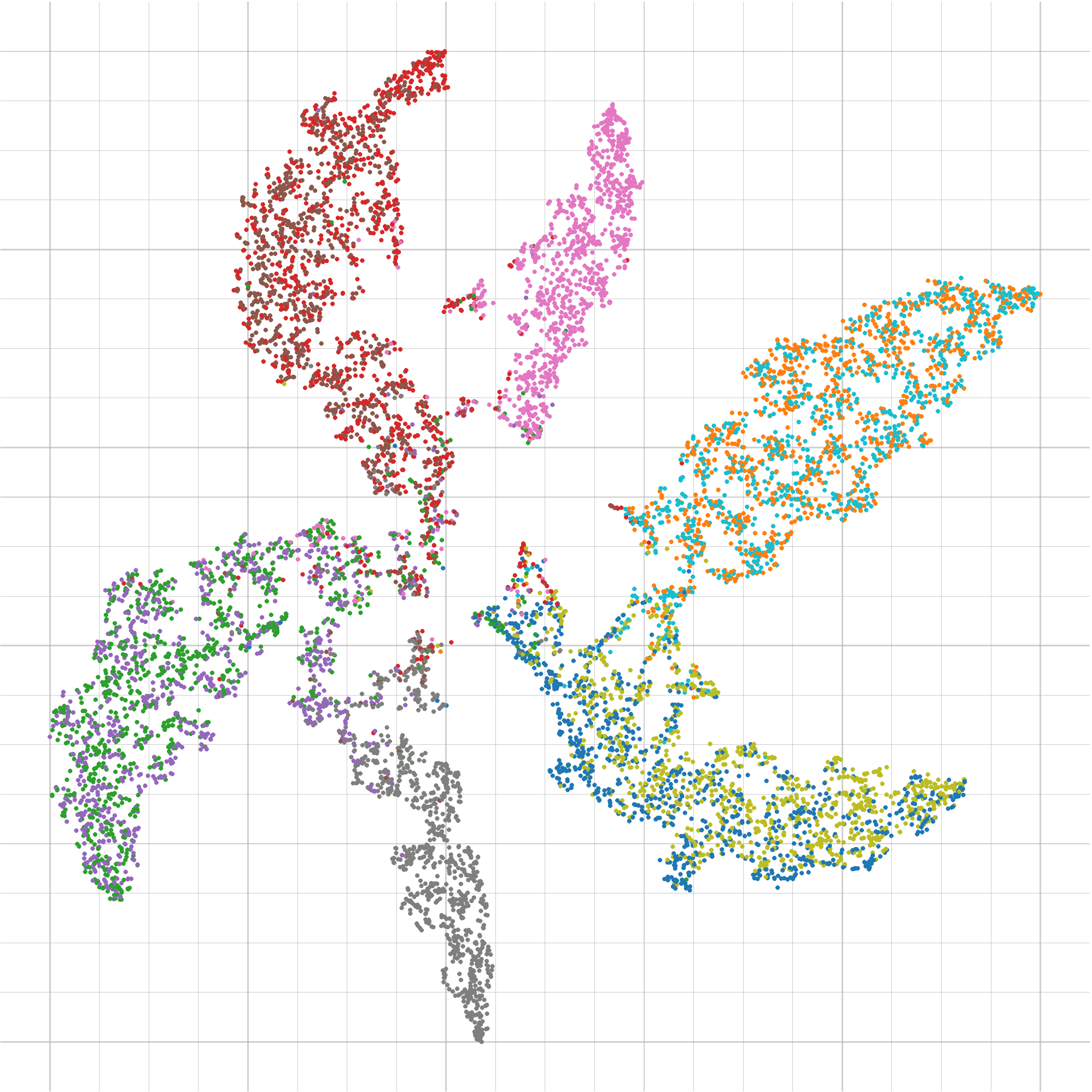}
     \caption{HashNet $64$ bits, \\
     $\frac{\min{\mathcal{D}_{\mathit{inter}}}}{\max{\mathcal{D}_{\mathit{intra}}}} = 0.8562$}
     \label{Fig.TSNE-a}
    \end{subfigure}
    \hfill
    \begin{subfigure}[b]{0.245\textwidth}
     \centering
     \includegraphics[width=\linewidth]{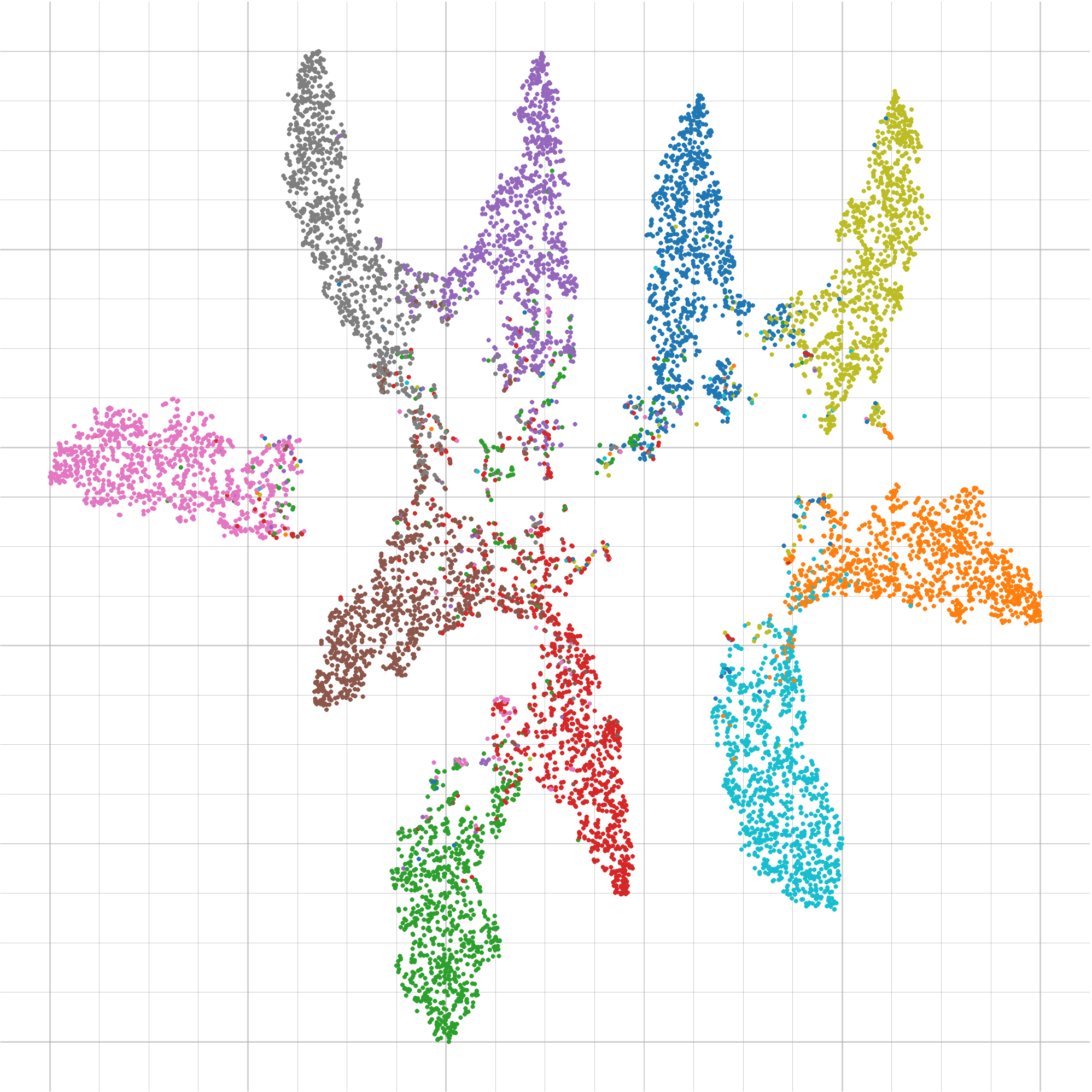}
     \caption{HashNet$_D$ $64$ bits, \\
     $\frac{\min{\mathcal{D}_{\mathit{inter}}}}{\max{\mathcal{D}_{\mathit{intra}}}} = 0.9334$}
     \label{Fig.TSNE-b}
    \end{subfigure}
    \hfill
    \begin{subfigure}[b]{0.245\textwidth}
     \centering
     \includegraphics[width=\linewidth]{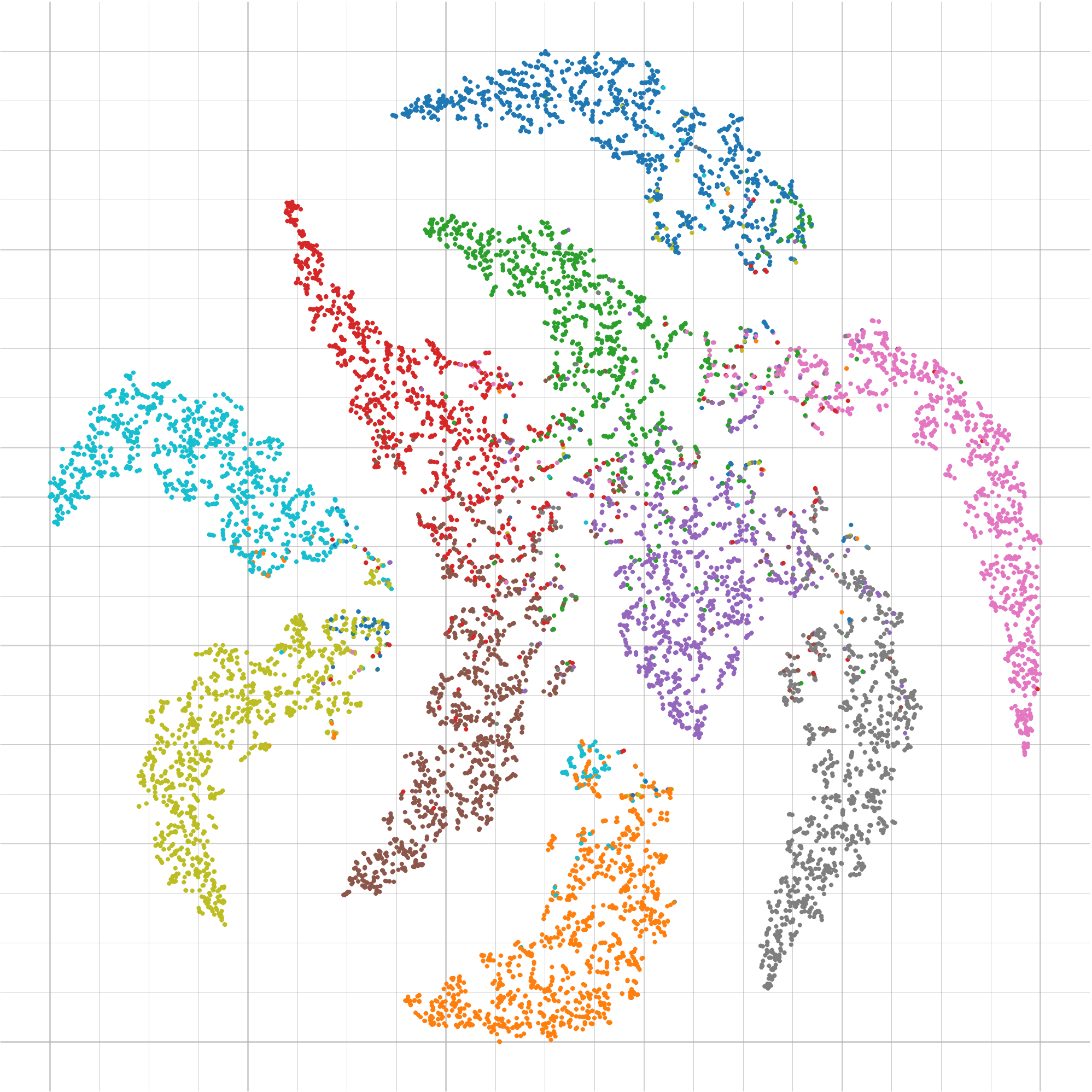}
     \caption{DBDH $64$ bits, \\
     $\frac{\min{\mathcal{D}_{\mathit{inter}}}}{\max{\mathcal{D}_{\mathit{intra}}}} = 0.9447$}
     \label{Fig.TSNE-c}
    \end{subfigure}
    \hfill
    \begin{subfigure}[b]{0.245\textwidth}
     \centering
     \includegraphics[width=\linewidth]{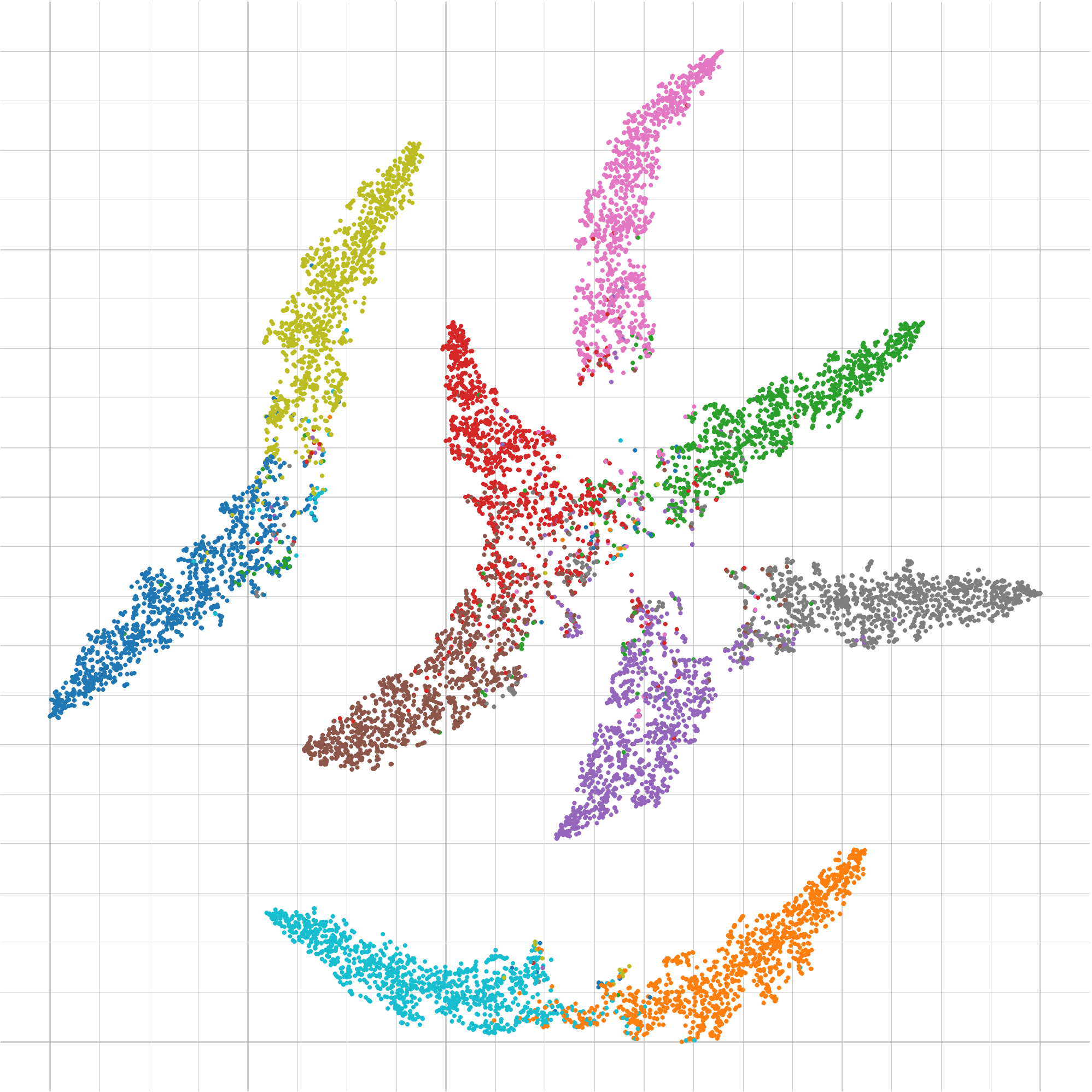}
     \caption{DBDH$_D$ $64$ bits, \\
     $\frac{\min{\mathcal{D}_{\mathit{inter}}}}{\max{\mathcal{D}_{\mathit{intra}}}} = 0.9761$}
     \label{Fig.TSNE-d}
    \end{subfigure}
    \caption{t-SNE visualization for HashNet and DBDH $\mathit{w/}$ and $\mathit{w/o}$ our integration, where the $\left(\cdot\right)_D$ variants have more compact structures and larger margin among different classes than original ones.}
    \label{Fig.TSNE}
\end{figure*}

\subsection{Qualitative Comparisons}
\textbf{t-SNE visualization.} To further analyze how our objective affects hash codes by adjusting $\mathcal{D}_{\mathit{inter}}$ and $\mathcal{D}_{\mathit{intra}}$, we conduct codes visualization by t-SNE~\cite{TSNE} on HashNet and DBDH. Specifically, we randomly sample $5,000$ $64$ bits codes from CIFAR-10 base split. t-SNE is then performed on these codes as shown in \cref{Fig.TSNE}, where points' color indicates class. From these figures, we give our humble explanation. Firstly, please look at \cref{Fig.TSNE-a}, codes extracted from the original HashNet are mixed among different classes. While in \cref{Fig.TSNE-c}, clusters are loose. These representations result in low $\min{\mathcal{D}_{\mathit{inter}}}$ and high $\max{\mathcal{D}_{\mathit{intra}}}$. When they are trained along with our method~(\cref{Fig.TSNE-b,Fig.TSNE-d}), we get more separable and distinct clusters compared to the original ones. Meanwhile, codes in a cluster are more compact than the original methods. Quantitative results of $\frac{\min{\mathcal{D}_{\mathit{inter}}}}{\max{\mathcal{D}_{\mathit{intra}}}}$ placed under figures also reveal this phenomenon.

\textbf{mAP \wrt $\frac{\min{\mathcal{D}_{\mathit{inter}}}}{\max{\mathcal{D}_{\mathit{intra}}}}$ visualization.} To further confirm the relationship between hashing performance and $\frac{\min{\mathcal{D}_{\mathit{inter}}}}{\max{\mathcal{D}_{\mathit{intra}}}}$, as in \cref{Sec.Ablation}, we plot mAP \wrt $\frac{\min{\mathcal{D}_{\mathit{inter}}}}{\max{\mathcal{D}_{\mathit{intra}}}}$ for all tested methods on $64$ bits ImageNet base split (\cref{Fig.AllMethodMAPGamma}). As the figure shows, most of the methods obtain higher mAP with larger $\frac{\min{\mathcal{D}_{\mathit{inter}}}}{\max{\mathcal{D}_{\mathit{intra}}}}$ and move from lower-left region to upper-right region (solid dots to translucent dots indicate original methods to integrated methods). Meanwhile, according to the positions of these dots, mAP and $\frac{\min{\mathcal{D}_{\mathit{inter}}}}{\max{\mathcal{D}_{\mathit{intra}}}}$ are supposed to be under linear relationship for different methods, \ie higher $\frac{\min{\mathcal{D}_{\mathit{inter}}}}{\max{\mathcal{D}_{\mathit{intra}}}}$ leads to higher mAP. Therefore, we give a regression line with $95\%$ confidence interval on the plot to confirm our observations. Extended from this, our proposed lower bound could also be a criterion of hash codes' performance.

\section{Conclusion}
\label{Sec.Conclusion}
In this paper, we conduct a comprehensive study on the characteristics of hash codes. As a result, we prove that hash codes' performance is lower-bounded by inter-class distinctiveness and intra-class compactness. Formulating such a lower bound as an objective, we could further lift it in hash-model training. Meanwhile, our proposed surrogate model for posterior estimation over hash codes' fully exploits the above guidance to perform low-bias model optimization and finally produce good codes. Extensive experiments conducted on three benchmark datasets confirm the effectiveness of our proposed method. We are able to boost current hash-models' performance with a flexible integration.

\section*{Limitation and Broader Impacts}
Hashing and related compact representation learning are able to significantly reduce computational requirements while improving memory efficiency when deploying to real scenarios. However, the main challenge for these techniques is information loss which results in a performance drop. To tackle this, our work on on hash codes' performance gives a lower bound, validated in two specific tasks. However, generalizing such lower bounds to various scenarios, \eg semi-supervised, unsupervised hash learning, are left for future study. Meanwhile, our posterior estimation approach is not verified for extremely long bits. Nevertheless, our work may still provide the potential to inspire researchers to improve not only hash learning, but also broader areas that adopt metric learning. Our study on posterior estimation may also help for precise discrete optimizations.

\begin{ack}
This work is supported by National Key Research and Development Program of China (No. 2018AAA0102200), the National Natural Science Foundation of China (Grant No. 62122018, No. 61772116, No. 62020106008, No. 61872064).
\end{ack}

\small{
\bibliographystyle{abbrvnat}
\bibliography{biblo}
}

\section*{Checklist}


\begin{enumerate}

\item For all authors...
\begin{enumerate}
  \item Do the main claims made in the abstract and introduction accurately reflect the paper's contributions and scope?
    \answerYes{}
  \item Did you describe the limitations of your work?
    \answerYes{See \cref{Sec.Conclusion}.}
  \item Did you discuss any potential negative societal impacts of your work?
    \answerNo{No explicit negative impacts.}
  \item Have you read the ethics review guidelines and ensured that your paper conforms to them?
    \answerYes{}
\end{enumerate}

\item If you are including theoretical results...
\begin{enumerate}
  \item Did you state the full set of assumptions of all theoretical results?
    \answerYes{See \cref{Sec.DisPar}.}
        \item Did you include complete proofs of all theoretical results?
    \answerYes{See \cref{Sec.DisPar}.}
\end{enumerate}

\item If you ran experiments...
\begin{enumerate}
  \item Did you include the code, data, and instructions needed to reproduce the main experimental results (either in the supplemental material or as a URL)?
    \answerYes{See supplemental material.}
  \item Did you specify all the training details (e.g., data splits, hyperparameters, how they were chosen)?
    \answerYes{See \cref{Sec.Exp}.}
        \item Did you report error bars (e.g., with respect to the random seed after running experiments multiple times)?
    \answerYes{See \cref{Sec.Exp}.}
        \item Did you include the total amount of compute and the type of resources used (e.g., type of GPUs, internal cluster, or cloud provider)?
    \answerYes{See \cref{Sec.Exp}.}
\end{enumerate}

\item If you are using existing assets (e.g., code, data, models) or curating/releasing new assets...
\begin{enumerate}
  \item If your work uses existing assets, did you cite the creators?
    \answerYes{See \cref{Sec.Exp}.}
  \item Did you mention the license of the assets?
    \answerYes{See \cref{Sec.Exp}.}
  \item Did you include any new assets either in the supplemental material or as a URL?
    \answerYes{Codes appear in supplementary materials.}
  \item Did you discuss whether and how consent was obtained from people whose data you're using/curating?
    \answerNA{}
  \item Did you discuss whether the data you are using/curating contains personally identifiable information or offensive content?
    \answerNA{}
\end{enumerate}

\item If you used crowdsourcing or conducted research with human subjects...
\begin{enumerate}
  \item Did you include the full text of instructions given to participants and screenshots, if applicable?
    \answerNA{}
  \item Did you describe any potential participant risks, with links to Institutional Review Board (IRB) approvals, if applicable?
    \answerNA{}
  \item Did you include the estimated hourly wage paid to participants and the total amount spent on participant compensation?
    \answerNA{}
\end{enumerate}

\end{enumerate}


\appendix

\section*{Appendix}
In this supplementary material, we discuss the following topics: Firstly, we give explanation in \cref{Sec.Definition} to demystify concepts of rank lists. Then, proofs of the propositions in main paper are given in \cref{Sec.Proof}. We further discuss why to adopt $\mathit{AP}$ as a criterion of hash codes' performance in \cref{Sec.Criterion}. To train hash-models, we treat the posterior of hash codes to be under the multivariate Bernoulli distribution. We explain why and how to perform posterior estimation in \cref{Sec.MVB,Sec.Est}. Additional experiments are finally given in \cref{Sec.AdditionExp}.

\section{Definitions}
\label{Sec.Definition}
\paragraph{The queries, true positives and false positives.} We demonstrate concepts of queries, true positives and false positives in rank lists in \cref{Fig.Ranklist,Fig.SwapBefore} for easy understanding. As the figure shows, any true positives or false positives are assigned with ranks $i$. Meanwhile, any true positives are also tagged by mis-ranks $m$ introduced in this paper. $m$ indicates how many false positives have the higher ranks than the current true positive.
\begin{figure}[H]
    \centering
    \includegraphics{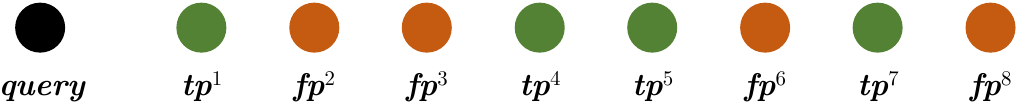}
    \caption{An example rank list for demonstration. True positives are green while false positives are orange. All positives have ranks $i$ placed on upper-right.}
    \label{Fig.Ranklist}
\end{figure}

We assume distances between query and any positive samples are different with each other. Then, we obtain following inequalities according to the property of a rank list:
\begin{equation}
    \label{Eq.DistanceOrder}
    d\left(\bm{q}, \bm{\mathit{tp}}^1\right) < d\left(\bm{q}, \bm{\mathit{fp}}^2\right) < \cdots < d\left(\bm{q}, \bm{\mathit{fp}}^8\right).
\end{equation}

\begin{figure}[H]
    \centering
    \begin{subfigure}[t]{\textwidth}
        \centering
        \includegraphics{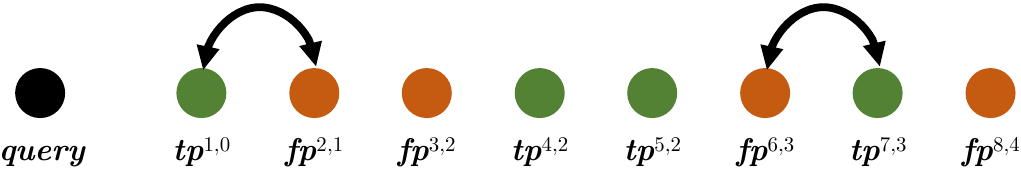}
        \caption{True positives and their mis-ranks (next to ranks of true positives). the operation of swaps could happen between two side-by-side samples due to distances change.}
        \label{Fig.SwapBefore}
    \end{subfigure}
    \hfill
    \begin{subfigure}[t]{\textwidth}
        \centering
        \includegraphics{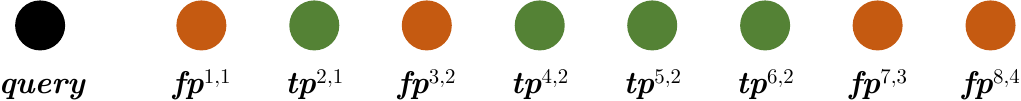}
        \caption{After swapping, the involved true positives will have altered ranks and mis-ranks, while other true positives' will remain unchanged.}
        \label{Fig.SwapAfter}
    \end{subfigure}
    \caption{Mis-ranks marked on true positives and swaps that change ranks and mis-ranks.}
\end{figure}

\paragraph{The side-by-side swaps.} Naturally, if a true positive and a false positive are placed side-by-side, and a swap happens between them due to the distances change, then mis-rank $m$ as well as rank $i$ of the true positive will be altered by $1$ (\cref{Fig.SwapBefore,Fig.SwapAfter}). Meanwhile, mis-ranks of any other true positives will remain unchanged. A special case is a side-by-side swap between two true positives or two false positives, where all ranks and mis-ranks would not change.

\paragraph{Normal swaps.} More generally, any swaps happen in a rank list would influence ranks and mis-ranks of involved positive samples. To determine results after swaps, we could decompose the normal swaps into a series of side-by-side swaps with involved positive samples.

\section{Proof}
\label{Sec.Proof}

\begin{corollary}
    Average precision increases $\mathit{iff} \; m$ decreases.
\end{corollary}
\begin{proof}
    From the above demonstration, if any swap happens in a rank list between true and false positives, the mis-rank of that true positive is definitely changed and will only result in increase or decrease of $m$ and $i$ by $1$. Correspondingly, precision at rank of involved true positive changes:
    \begin{equation*}
        \frac{i \pm 1 - \left( m \pm 1\right)}{i \pm 1} = \frac{i - m}{i \pm 1},
    \end{equation*}
    which is reversely proportional to $m$. Since precision of other true positives' are unchanged, $\mathit{AP}$ is therefore influenced by the above true positive and reversely proportional to $m$. According to the derivation of normal swaps in \cref{Sec.Definition}, this corollary will cover all cases of $\mathit{AP}$ calculations.
\end{proof}

\begin{proposition}
    \begin{equation*}
        \barabove{m} \propto \frac{\max{d\left(\bm{q}, \bm{\mathit{tp}}\right)}}{\min{d\left(\bm{q}, \bm{\mathit{fp}}\right)}} \; \forall \bm{\mathit{tp}} \in \bm{\mathit{TP}}, \; \bm{\mathit{fp}} \in \bm{\mathit{FP}}
    \end{equation*}
    where $\baraboveshort{\cdot}$ denotes upper bounds. Correspondingly,
    \begin{equation*}
        \barbelow{\mathit{AP}} \propto \frac{\min{d\left(\bm{q}, \bm{\mathit{fp}}\right)}}{\max{d\left(\bm{q}, \bm{\mathit{tp}}\right)}} \; \forall \bm{\mathit{tp}} \in \bm{\mathit{TP}}, \; \bm{\mathit{fp}} \in \bm{\mathit{FP}}
    \end{equation*}
    where $\barbelowshort{\cdot}$ denotes lower bounds.
\end{proposition}
\begin{proof}
    Considering any true positive, we are able to derive following inequalities from it according to \cref{Eq.DistanceOrder}:
    \begin{equation}
        \label{Eq.UpperBound}
        \min{d\left(\bm{q}, \bm{\mathit{fp}}\right)}_{m} < d\left(\bm{q}, \bm{\mathit{tp}}^{m}\right) < \min{d\left(\bm{q}, \bm{\mathit{fp}}\right)}_{m + 1}
    \end{equation}
    where $\min\left(\cdot\right)_{m}$ is the $m$-th minimum value among the whole set. For example, from \cref{Fig.Ranklist}, we could obtain:
    \begin{equation}
        \label{Eq.LowerBound}
        \min{d\left(\bm{q}, \bm{\mathit{fp}}\right)}_3 <  d\left(\bm{q}, \bm{\mathit{tp}}^{7,3}\right) < \min{d\left(\bm{q}, \bm{\mathit{fp}}\right)}_4
    \end{equation}
    where $\min{d\left(\bm{q}, \bm{\mathit{fp}}\right)}_3 = d\left(\bm{q}, \bm{\mathit{fp}}^6\right)$ and $\min{d\left(\bm{q}, \bm{\mathit{fp}}\right)}_4 = d\left(\bm{q}, \bm{\mathit{fp}}^8\right)$.

    Then, for the last true positive in rank list, which has the largest $m$ among all true positives, \ie $\barabove{m}$, according to above inequalities, we could arrive:
    \begin{equation*}
        \min{d\left(\bm{q}, \bm{\mathit{fp}}\right)}_{\barabove{m}} < \max{d\left(\bm{q}, \bm{\mathit{tp}}\right)} < \min{d\left(\bm{q}, \bm{\mathit{fp}}\right)}_{\barabove{m} + 1}.
    \end{equation*}

    Now, if $\max{d\left(\bm{q}, \bm{\mathit{tp}}\right)}$ becomes smaller and $\min{d\left(\bm{q}, \bm{\mathit{fp}}\right)}_{\barabove{m}}$ becomes larger so that a swap happens between the last true positive and its left most false positive, then we will have:
    \begin{equation*}
    \begin{split}
        \barabove{m}' &= \barabove{m} - 1, \; \mathit{iff} \\
        \max{d\left(\bm{q}, \bm{\mathit{tp}}\right)}' &< \max{d\left(\bm{q}, \bm{\mathit{tp}}\right)},\\
        \min{d\left(\bm{q}, \bm{\mathit{tp}}\right)}_{\barabove{m}}' &> \min{d\left(\bm{q}, \bm{\mathit{tp}}\right)}_{\barabove{m}} \;\mathit{and}\\
        \min{d\left(\bm{q}, \bm{\mathit{tp}}\right)}_{\barabove{m}}' &> \max{d\left(\bm{q}, \bm{\mathit{tp}}\right)}'
    \end{split}
    \end{equation*}
    where $\left(\cdot\right)'$ indicates the new value. Combining our assumption in \cref{Sec.Definition} as well as the pigeonhole principle, increase of $\min{d\left(\bm{q}, \bm{\mathit{tp}}\right)}$  results in increase of $\min{d\left(\bm{q}, \bm{\mathit{tp}}\right)}_{\barabove{m}}$. Therefore:
    \begin{equation*}
        \left.\begin{aligned}
            \max{d\left(\bm{q}, \bm{\mathit{tp}}\right)}\downarrow \\
            \min{d\left(\bm{q}, \bm{\mathit{tp}}\right)}\uparrow \Rightarrow \min{d\left(\bm{q}, \bm{\mathit{tp}}\right)}_{\barabove{m}}\uparrow
        \end{aligned}\right\rbrace \barabove{m}\downarrow,
    \end{equation*}
    and vice versa. This proves \cref{Eq.UpperBound}. \cref{Eq.LowerBound} is immediately obtained since $\mathit{AP}$ is reversely proportional to $\barabove{m}$ according to above corollary.
\end{proof}

\begin{proposition}
\begin{equation}
    \barbelow{\mathit{AP}} \propto \frac{\min{d\left(\bm{q}, \bm{\mathit{fp}}\right)}}{\max{d\left(\bm{q}, \bm{\mathit{tp}}\right)}} \geq \mathit{const} \cdot \frac{\min{\mathcal{D}_{\mathit{inter}}}}{\max{\mathcal{D}_{\mathit{intra}}}}.
\end{equation}
\end{proposition}
\begin{proof}
    Without loss of generality, $d\left(\bm{q}, \bm{\mathit{tp}}\right)$ is covered by distances between any two codes of the same class $c$ (denoted by $\mathcal{D}^c$): $d\left(\bm{q}, \bm{\mathit{tp}}\right) \subseteq \mathcal{D}^c = \left\{d\left(\bm{b}^i, \bm{b}^j\right), \; \mathit{where}\; y^i = y^j = c\right\}$. In contrast, $d\left(\bm{q}, \bm{\mathit{fp}}\right)$ is covered by: $d\left(\bm{q}, \bm{\mathit{fp}}\right) \subseteq \mathcal{D}^{\neq c} = \left\{d\left(\bm{b}^i, \bm{b}^j\right), \; \mathit{where}\; y^i = c,\; y^j \neq c\right\}$. For the first case, we have:
    \begin{equation*}
    \begin{split}
        d\left(\bm{b}^i, \bm{c}^c\right) \leq \max{\mathcal{D}_{\mathit{intra}}}, &\; d\left(\bm{b}^j, \bm{c}^c\right) \leq \max{\mathcal{D}_{\mathit{intra}}},\\
        \mathit{for\; any}\; y^i &= y^j = c
    \end{split}
    \end{equation*}
    where $\bm{c}^c$ is the center of class $c$. Then, according to the triangle inequality:
    \begin{equation*}
        d\left(\bm{b}^i, \bm{b}^j\right) \leq d\left(\bm{b}^i, \bm{c}^c\right) + d\left(\bm{b}^j, \bm{c}^c\right) \leq 2\max{\mathcal{D}_{\mathit{intra}}}.
    \end{equation*}
    Therefore,
    \begin{equation*}
        \max{d\left(\bm{q}, \bm{\mathit{tp}}\right)} \subseteq \mathcal{D}^c \leq 2\max{\mathcal{D}_{\mathit{intra}}} = \mathit{const} \cdot \max{\mathcal{D}_{\mathit{intra}}}.
    \end{equation*}
    Similarly, $\min{d\left(\bm{q}, \bm{\mathit{fp}}\right)} \geq \mathit{const} \cdot \min{\mathcal{D}_{\mathit{inter}}}$ could be derived. Combined with two inequalities, the above proposition is proved.
\end{proof}

\subsection{Analysis on the Proposed Lower Bound}
\subsubsection{Is the Introduced Lower Bound Tight?}
\label{Sec.Tight}
To determine whether the lower bound is tight is a little bit difficult. We firstly introduce some concepts and assumptions to make it easier. Let us start at the example placed in beginning of \cref{Sec.Definition}.

\begin{assumption}
    Any positive samples do not have the same distances to query. This ensures \cref{Eq.DistanceOrder}.
\end{assumption}
\begin{assumption}
    Noticed that we are working in the Hamming space where the Hamming distances between any two samples are discrete and range from $0$ to $h$, \cref{Eq.DistanceOrder} becomes:
    \begin{equation}
        0 \leq d\left(\bm{q}, \bm{\mathit{tp}}^1\right) < d\left(\bm{q}, \bm{\mathit{fp}}^2\right) < \cdots < d\left(\bm{q}, \bm{\mathit{fp}}^8\right) \leq h.
    \end{equation}
\end{assumption}
\begin{assumption}
    The above array is strictly no gaps \ie differences between any side-by-side $d\left(\bm{q}, \cdot\right)$ are $1$.
\end{assumption}

Then, we would derive the closed form lowest $\mathit{AP}$ with $\frac{\min{d\left(\bm{q}, \bm{\mathit{fp}}\right)}}{\max{d\left(\bm{q}, \bm{\mathit{tp}}\right)}}$ under the same order of magnitude. This conclusion could be intuitively drawn from the above example where $\bm{\mathit{tp}}^{7,3}$ and $\bm{\mathit{fp}}^2$ determine $\max{\left(\bm{q}, \bm{\mathit{tp}}\right)}, \min{\left(\bm{q}, \bm{\mathit{fp}}\right)}$. If we keep two values unchanged (in other words, ranks of $\bm{\mathit{tp}}^{7,3}$ and $\bm{\mathit{fp}}^2$ unchanged), then the highest $\mathit{AP}$ will appear as the following figure shows:
\begin{figure}[H]
    \centering
    \includegraphics{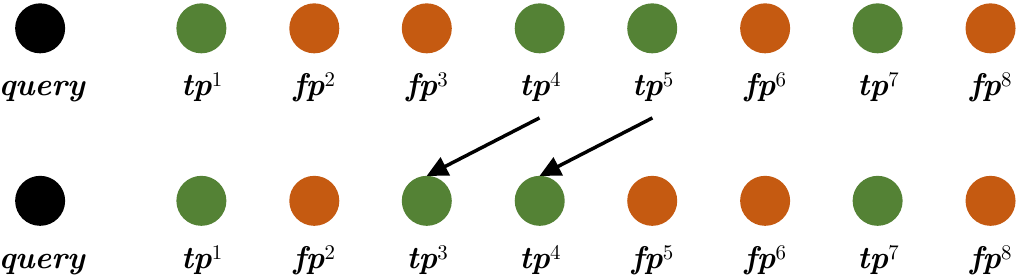}
    \caption{When $\max{\left(\bm{q}, \bm{\mathit{tp}}\right)}, \min{\left(\bm{q}, \bm{\mathit{fp}}\right)}$ are fixed, the highest $\mathit{AP}$ will appear when all true positives have higher ranks than false positives in-between $\max{\left(\bm{q}, \bm{\mathit{tp}}\right)}, \min{\left(\bm{q}, \bm{\mathit{fp}}\right)}$.}
    \label{Fig.HighAP}
\end{figure}
And the lowest $\mathit{AP}$ will appear as the following figure shows:
\begin{figure}[H]
    \centering
    \includegraphics{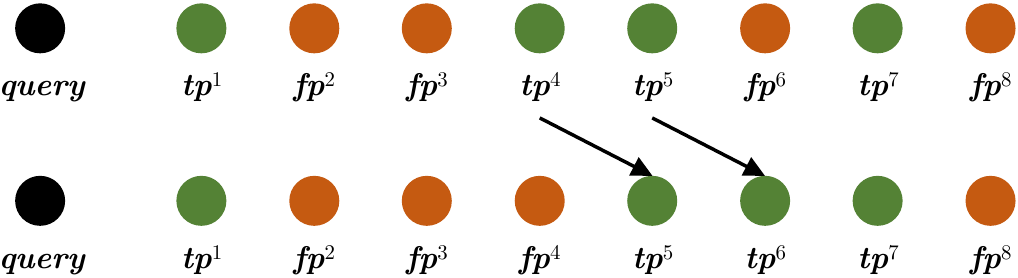}
    \caption{Correspondingly, the lowest $\mathit{AP}$ will appear when all true positives have lower ranks than false positives in-between $\max{\left(\bm{q}, \bm{\mathit{tp}}\right)}, \min{\left(\bm{q}, \bm{\mathit{fp}}\right)}$.}
    \label{Fig.LowAP}
\end{figure}

Based on this, we could easily derive that the lowest AP equals to
\begin{equation*}
\min{d\left(\bm{q}, \bm{\mathit{fp}}\right)} - 1 + \sum_{i=1}^{\lvert \bm{\mathit{TP}} \rvert}{\frac{i}{\max{d\left(\bm{q}, \bm{\mathit{tp}}\right)} - \min{d\left(\bm{q}, \bm{\mathit{fp}}\right)} + i}}
\end{equation*}
which is proportional to $\frac{\min{d\left(\bm{q}, \bm{\mathit{fp}}\right)}}{\max{d\left(\bm{q}, \bm{\mathit{tp}}\right)}}$. Therefore, the lower bound is tight.

We could further extend $\frac{\min{d\left(\bm{q}, \bm{\mathit{fp}}\right)}}{\max{d\left(\bm{q}, \bm{\mathit{tp}}\right)}}$ to $\frac{\min{\mathcal{D}_{\mathit{inter}}}}{\max{\mathcal{D}_{\mathit{intra}}}}$ for the tight lower bound. From \cref{Eq.Hero}, we find that $\frac{\min{d\left(\mathbf{q}, \mathbf{fp}\right)}}{\max{d\left(\mathbf{q}, \mathbf{tp}\right)}} \geq \frac{\min{\mathcal{D}_\mathit{inter}}}{\max{\mathcal{D}_\mathit{intra}}}$. The equality is achieved when query's code $\mathbf{q}$ is exactly the same as its class-center $\mathbf{c}$'s code. In this condition, the tight lower bound is derived by $\frac{\min{\mathcal{D}_\mathit{inter}}}{\max{\mathcal{D}_\mathit{intra}}}$.

Then, could it be applied to general cases? We give our humble discussion for a simple study. There may have untouched complicated cases that will be leaved for future study.

\paragraph{Case 1: Duplicated positives.} If some samples are hashed to the same binary code (collision), then distances from query to them are all equal. They will appear at the same position of the rank list. If they are all true positives or false positives, then we could directly treat them as a single duplicated sample and follow the above rules. For example:
\begin{figure}[H]
    \centering
    \includegraphics{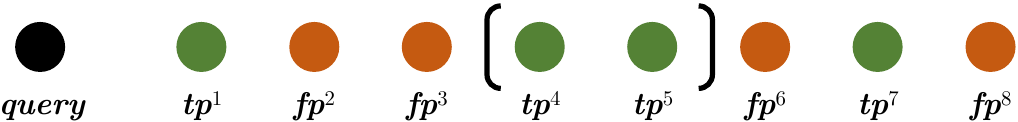}
\end{figure}

The above rank list has the duplicated true positives ($d\left(\bm{q}, \bm{\mathit{tp}}^{4,2}\right) = d\left(\bm{q}, \bm{\mathit{tp}}^{5,2}\right)$). If a swap happens between them and $\bm{\mathit{fp}}^6$, the rank list will become:
\begin{figure}[H]
    \centering
    \includegraphics{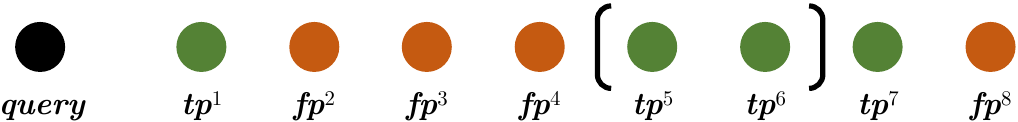}
\end{figure}

where $i, m$ of the duplicated true positives are both increased by $1$. Obviously, the lower bound is still tight.

\paragraph{Case 2: Mixed positives.} It is tricky when true positives and false positives have the same distance with query (we call them mixed positives). The sorting algorithm to produce final rank list also has impact to determine ranks of these mixed positives. Whether the lower bound is tight in this case is hard to judge, but our lower bound is still valid since above corollary still makes sense.

\paragraph{Case 3: Rank list with gaps.} If there are gaps in between two distances, \eg, $d\left(\bm{q}, \bm{\mathit{fp}}^i\right) \ll d\left(\bm{q}, \bm{\mathit{tp}}^{i+1}\right)$ (this usually happens on outliers), then $\mathit{AP}$ will be far from the lower bound. Please refer to \cref{Sec.Edge} for details.

In conclusion, under the above assumptions Asm. 1, 2, 3, our lower bound is tight. Meanwhile, our lower bound also covers common cases as in the above discussion and makes a strong connection to $\mathit{AP}$.

\subsection{Edge Cases of the Lower Bound}
\label{Sec.Edge}
Some edge cases when the lower bound is far way from the $\mathit{AP}$ are further revealed according to the above analysis.

\paragraph{Edge Case 1.} A huge amount of samples are hashed to the same binary code. Based on Case 1 and 2 in \cref{Sec.Tight}, we will observe many duplicated or mixed samples. Their ranks will be increased / decreased simultaneously, and AP significantly changes along with them. The rank list is now be treated as \enquote{unstable}. Intuitively, when a hash-model has poor performance, it could not distinguish differences between samples and simply hashes them to the same code, and such a model will produce \enquote{unstable} rank lists.

\paragraph{Edge Case 2.} Gaps in rank list. In the above example, if $d\left(\bm{q}, \bm{\mathit{fp}}^6\right) \ll d\left(\bm{q}, \bm{\mathit{tp}}^{7, 3}\right) = \max{d\left(\bm{q}, \bm{\mathit{TP}}\right)}$, $\max{d\left(\bm{q}, \bm{\mathit{TP}}\right)}$ needs to be significantly decreased until a swap happens between $\bm{\mathit{fp}}^6$ and $\bm{\mathit{tp}}^{7, 3}$ to influence $\mathit{AP}$. Therefore, $\mathit{AP}$ is high but $\frac{\min{\left(\bm{q}, \bm{\mathit{fp}}\right)}}{\max{\left(\bm{q}, \bm{\mathit{tp}}\right)}}$ is low. A potential reason leads to this edge case is outliers.

\section{Criterion of Hash Codes' Performance}
\label{Sec.Criterion}
Hashing techniques are widely applied in common multimedia tasks. To determine hash codes' performance, the direct way is to adopt evaluation metrics used in these tasks \eg retrieval precision, recall or recognition accuracy. In this section, we give the detailed explanation on how to adopt $\mathit{AP}$ to cover the above typical metrics.

\paragraph{Retrieval.} Common evaluation metrics in retrieval can be derived from $\mathit{AP}$. Specifically,

\setlist[itemize]{leftmargin=*}
\begin{itemize}
  \item Precision at rank $i$ equals to $\frac{i - m}{i}$. The corollary in \cref{Sec.Proof} exactly applies to it.
  \item Recall at rank $i$ equals to $\frac{i - m}{\lvert \bm{T} \rvert}$ where $\bm{T}$ is set of all groundtruth samples. Therefore, $\lvert \bm{T} \rvert$ is a constant and recall increases $\mathit{iff} \; m$ decreases.
  \item F-score equals to $\frac{2}{\mathit{recall}^{-1} + \mathit{precision}^{-1}} = \frac{2}{i / \left(i - m\right) + \lvert \mathbf{T} \rvert / \left(i - m\right)} = \frac{2\left(i - m\right)}{i + \lvert \mathbf{T} \rvert}$.
\end{itemize}

All the above metrics are reversely proportional to $m$. Therefore, analysis in \cref{Sec.Proof} is also valid to them.

\paragraph{Recognition.} There are many ways to adopt hash codes for recognition, and we discuss two main approaches here. 1) $k$NN based recognition is very similar to retrieval. In this scenario, prediction is obtained by voting among top-$k$ retrieved results. If there are more than half true positives in rank list, then the prediction will be correct, otherwise it will be wrong. Therefore, the recognition accuracy can be treated as a special case where we hold a rank list of all samples and expect $m < \nicefrac{k}{2}$ when $i = k$, which is induced to the goal of increasing $\mathit{AP}$.

As for logistic regression models for recognition, a learnable weight is adopted to make predictions: $\bm{b} \xrightarrow{\bm{w}} \bm{\mathit{cls}},\; \bm{w} \in \mathbb{R}^{C \times h}$ where the output $\bm{\mathit{cls}}$ is $C$-dim scores for each class. We could treat $\bm{w}$ as $C$ class-prototypes. Each prototype $\bm{w}^i$ is a $h$-dim vector belongs to class $i$ while $\bm{\mathit{cls}}_i$ measures inner-product similarity between prototype and hash code $\bm{b}$. Inner-product similarity is an approximation of Hamming distance and $\bm{w}$ is trained by cross-entropy objective, resulting in maximizing $\bm{\mathit{cls}}_c$ while minimizing $\bm{\mathit{cls}}_{\mathit{other}}$. Such a goal is equivalent to maximizing intra-class compactness and inter-class distinctiveness.

\subsection{Potential Use Cases of the Lower Bound}
Exploring use cases of the proposed lower bound would reveal significance and value of this work. Since Figure 1 and 5 in main paper and the above analysis tell us the value of $\frac{\min{\mathcal{D}_\mathit{inter}}}{\max{\mathcal{D}_\mathit{intra}}}$ partially reflects hash codes' performance, we would quickly evaluate model's performance by such a metric other than calculating $\mathit{AP}$ or accuracy which is time consuming. Therefore, adopting it as a performance indicator benefits for model tuning or selection, including but not limited to parameter search.

\section{Multivariate Bernoulli Distribution}
\label{Sec.MVB}
In this paper, we treat the posterior of hash codes $p\left(\bm{\mathcal{B}} \mid \bm{\mathcal{X}}\right)$ is under Multivariate Bernoulli distribution $\MVBer$. This is based on our observations of hash-model outputs. To demonstrate this claim, a toy $2$ bits CSQ model is trained on $2$ randomly chosen classes of CIFAR-10. We then plot $\bm{\ell}$ extracted by CSQ in \cref{Fig.Demo}.

\begin{figure}[H]
    \centering
    \begin{minipage}[b][][b]{0.42\textwidth}
        \centering
        \includegraphics[width=0.75\linewidth]{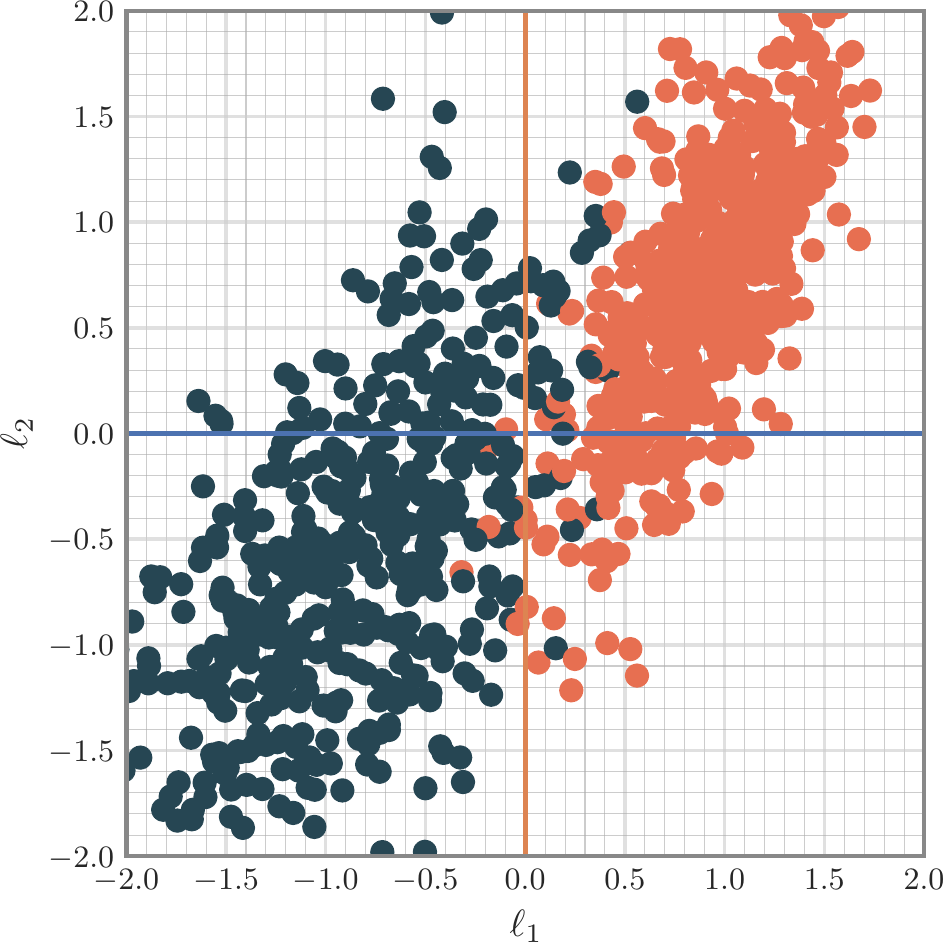}
        \caption{$\bm{\ell}$ extracted from a toy $2$ bits CSQ model with $2$ classes' samples of CIFAR-10. Correlations exist in $\bm{b}_1, \;\bm{b}_2$.}
        \label{Fig.Demo}
    \end{minipage}
    \hfill
    \begin{minipage}[b][][b]{0.52\textwidth}
        \centering
        \begin{tabular}{@{}|cc|cc|@{}}\toprule
                       & $p\left(\bm{b}_1\right)$ & $\bm{b}_1 < 0$ & $\bm{b}_1 > 0$  \\
        $p\left(\bm{b}_2\right)$   &            & $0.473$        & $0.527$         \\ \midrule
        $\bm{b}_2 > 0$ & \multicolumn{1}{c|}{$0.525$}    & \multicolumn{1}{c|}{$0.079$}        & \multicolumn{1}{c|}{$0.446$}         \\ \cmidrule(l){3-4}
        $\bm{b}_2 < 0$ & \multicolumn{1}{c|}{$0.475$}    & \multicolumn{1}{c|}{$0.394$}        & \multicolumn{1}{c|}{$0.081$}         \\ \bottomrule
        \end{tabular}
        \caption{Estimated joint and marginal probabilities. Obviously, joint probability is not equal to product of marginal probability.}
        \label{Tab.Demo}
    \end{minipage}
\end{figure}

Obviously, features are concentrated on the upper-right and lower-left regions on the figure. This indicates that hash codes of these samples are tend to be $\left(\text{+1+1}\right)$ or $\left(\text{-1-1}\right)$ while very few samples are hashed to $\left(\text{+1-1}\right)$ or $\left(\text{-1+1}\right)$. Based on this pattern, we could infer that if $\bm{b}_1$ and $\bm{b}_2$ have a positive correlation. We also count frequencies of all $4$ possible code combinations and calculate the estimated joint and marginal probability in \cref{Tab.Demo}, which confirms our observation. As the figure and the table show, joint probabilities are not equal to product of marginal probabilities, meaning that correlation exists among codes. Therefore, $\MVBer$ is suitable to model the posterior of hash codes. Previous works are not sufficient to model it since they process each hash bit independently.

For instance, if $h = 4$, joint probability of $\bm{b} = \left(\text{+1+1-1-1}\right)$ is formulated as:
\begin{equation}
\label{Eq.Demo}
    p\left(\bm{b} = \left(\text{+1+1-1-1}\right)\right) = p\left(\bm{b}_1<0,\bm{b}_2<0,\bm{b}_3>0,\bm{b}_4>0\right).
\end{equation}
Unfortunately, estimating this joint probability is difficult~\cite{MVB3,MVB1,MVB2}. Furthermore, all $2^h$ joint probabilities are required to be estimated in order to control hash codes precisely. Therefore, we propose our posterior estimation approach to tackle above difficulties.

\section{Demonstration of Posterior Estimation}
\label{Sec.Est}
We continue to use the example in \cref{Eq.Demo} to explain how we build our surrogate model to perform posterior estimation. Firstly, $4$ bits hash codes have all $16$ choices from $\left(\text{-1-1-1-1}\right)$ to $\left(\text{+1+1+1+1}\right)$. Therefore, $\bm{o}$ has $16$ entries where each entry in $\bm{o}$ will be used to estimate a specific joint probability. We provide a simple demonstration in \cref{Fig.Posterior} to show how we process it.

\begin{figure}[H]
    \centering
    \includegraphics{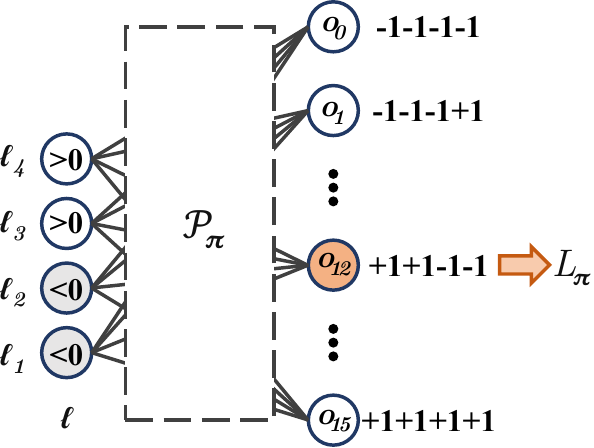}
    \caption{Demonstration of how $\mathcal{P}_{\bm{\pi}}$ performs posterior estimation.}
    \label{Fig.Posterior}
\end{figure}

Specifically, if we feed $\bm{\ell}$ where $\bm{\ell}_1 < 0,\;\bm{\ell}_2 < 0,\;\bm{\ell}_3 > 0,\;\bm{\ell}_4 > 0$ \ie $\bm{b} = \left(\text{+1+1-1-1}\right)$, then we could directly find output $\bm{o}_{12} \estimates p\left(\bm{b}_1 < 0, \bm{b}_2 < 0, \bm{b}_3 > 0, \bm{b}_4 > 0\right)$ where $12 = 1 \cdot 2^3 + 1 \cdot 2^2 + 0 \cdot 2^1 + 0 \cdot 2^0$ and perform MLE on it to train model $\pi$.

\subsection{Implementation of Posterior Estimation}
Correspondingly, we provide a minimal PyTorch-style implementation of $\mathcal{P}_{\bm{\pi}}$ and its training objective in \cref{Fig.Code}. We first give one of a implementation of non-linear $\bm{\ell}$ to $\bm{o}$ mapping in line \textnumero~\texttt{\ref{Code.NonLinearStart}$\sim$\ref{Code.NonLinearEnd}}. Then, the main body of surrogate model (line \textnumero~\texttt{\ref{Code.SurrogateStart}$\sim$\ref{Code.SurrogateEnd}}) shows how we process codes in blocks (line \textnumero~\texttt{\ref{Code.BlockCode}}) and how we convert every $8$ bits inputs to $256$ categorical outputs and calculate loss \wrt targets (line \textnumero~\texttt{\ref{Code.Loss}}).

\begin{figure*}[!h]
\centering
\begin{code}
# One of the implementations for non-linear  |\label{Code.NonLinearStart}|
# mapping to be used in surrogate model.
_nonLinearNet = lambda: nn.Sequential(
    nn.Linear(8, 256),
    nn.SiLU(),
    nn.Dropout(0.5),
    nn.Linear(256, 256)
)    |\label{Code.NonLinearEnd}|

class Surrogate(nn.Module):  |\label{Code.SurrogateStart}|
    """
    Surrogate model to estimate MVB and further provide gradients.

    Args:
        bits (int): Length of hash codes.
    """
    def __init__(self, bits: int):
        super().__init__()
        # Number of blocks.
        self._u = bits // 8
        self._bits = bits
        self._net = nn.ModuleList(
            _nonLinearNet() for _ in range(self._u)
        )
        # binary to decimal multiplier
        self.register_buffer("_multiplier",
            (2 ** torch.arange(8)).long())
        self._initParameters()

    def forward(self, x: Tensor, t: Tensor = None) -> (Tensor):
        """
        Module forward.
        Args:
            x (Tensor): [N, h] Model outputs before `.sign()`
                (un-hashed values).
            t (Tensor, Optional): [N, h] Target hash result of x.
        Return:
            Tensor: [] If t is None, produces loss to
                train surrogate model. Otherwise,
                produces loss to calculate surrogate
                gradient w.r.t. x.
        """
        loss = list()
        t = t or x
        # split hash codes into `u` pieces, each 8 bits.
        for subNet, org, tgt in zip(self._net,
            torch.chunk(x, self._u, -1),
            torch.chunk(t, self._u, -1)): |\label{Code.BlockCode}|
            # feed each 8 bits codes into i-th non-linear model
            # to get [256] predictions.
            prd = subNet(org)
            y = ((tgt > 0) * self._multiplier).sum(-1)  |\label{Code.Loss}|
            loss.append(F.cross_entropy(prd, y))
        return sum(loss)  |\label{Code.SurrogateEnd}|
\end{code}
\caption{Minimal implementation of our proposed surrogate model for posterior estimation.}
\label{Fig.Code}
\end{figure*}

\begin{figure*}[t]
    \centering
    \includegraphics[width=\linewidth]{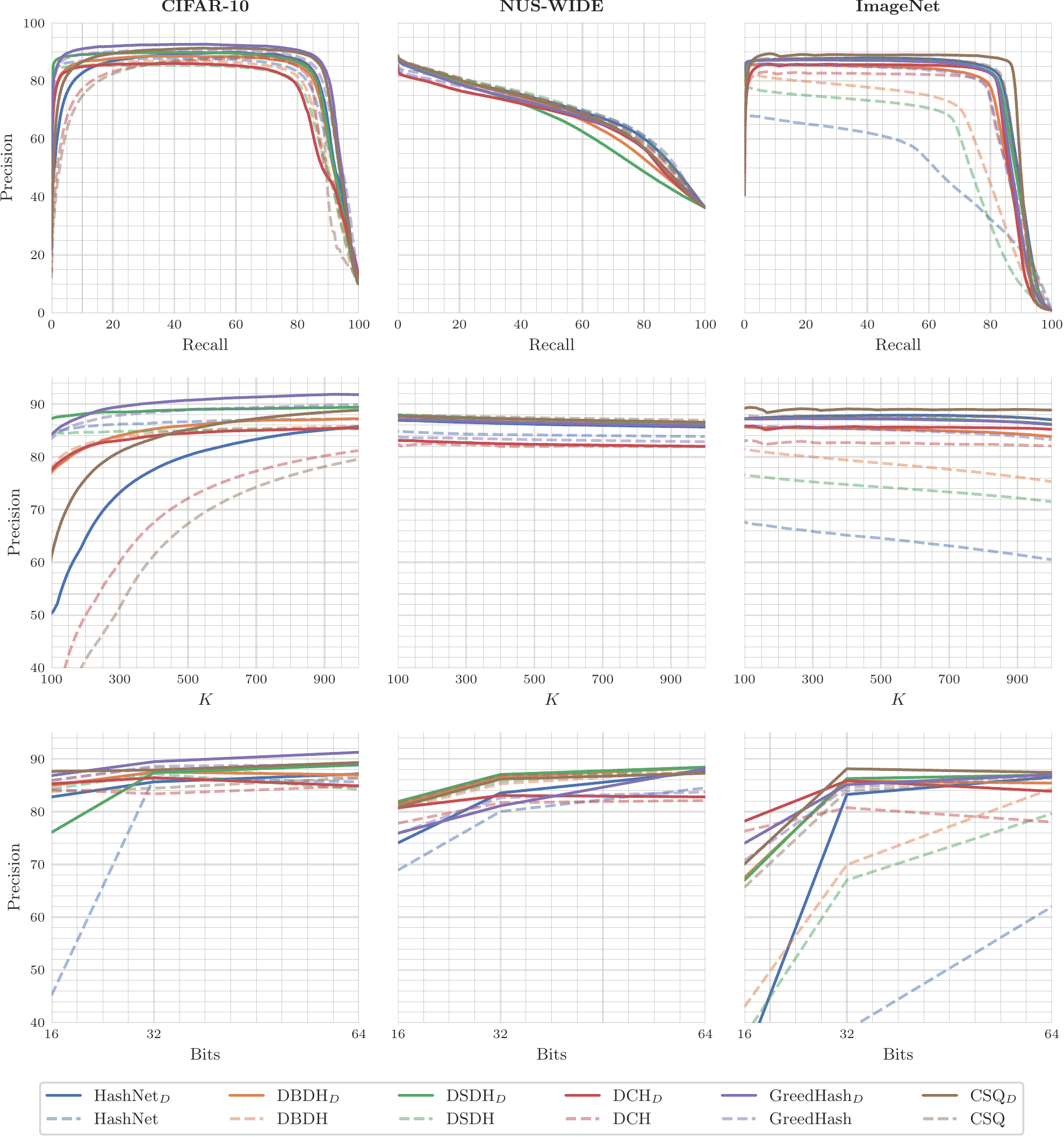}
    \caption{\textbf{First row}: Precision-Recall (P-R) curves for three datasets on $64$ bits. With our method integrated, most of methods' P-R curves slightly move to upper right. \\
    \textbf{Second row}: Precision@$K$ (P@$K$, $K$ from $100$ to $1,000$) curves for three datasets on $64$ bits. With our method integrated, most of methods' P@$K$ curves slightly go up. \\
    \textbf{Third row}: Precision@$H = 2$ (retrieval inside the $H = 2$ Hamming ball) \wrt code-length curves on three datasets.}
    \label{Fig.AllCurves}
\end{figure*}

\section{Additional Experiments}
\label{Sec.AdditionExp}
We continue to conduct experiments for more comparisons and ablations. Before this, we firstly provide more implementation details as a supplement to main paper.

As \cref{Fig.Code} shows, we choose a normal two layer model with SiLU activation as $\mathcal{P}_{\bm{\pi}}$ to perform posterior estimation. To avoid over-fitting, we insert a dropout layer with $0.5$ probability between layers. We train all methods for $100$ epoches with batch-size $64$, and learning rate is exponential decayed by $0.1$ every $10$ epoches. Most of the methods could achieve $95\%$ performance after $30$ epoches. We do not perform grid-search on hyper-parameters to obtain the highest results since it is not our topic in this paper.

To generate centers of specific classes for training (Alg.1 Line.2), we want a \textit{perfect} generator which could produce finite number of codes with maximized pairwise Hamming distance among each other. Such a problem is formulated as the $A_q\left(n, d\right)$ problem~\cite{aqnd-problem1,aqnd-problem2} and such perfect codes are recognized to attain the \textit{Hamming bound}~\cite{HammingBound3,HammingBound1,HammingBound2}. However, this is extremely hard to tackle. Therefore, we conduct ablation study on a series of error-correction codes which aims at approaching the above goal to pick the best one, which will be explained below.

\subsection{Performance on Large Dataset}
To evaluate performance of our proposed method on large-scale data, we conduct a new experiment on the whole ImageNet (ILSVRC 2012). Specifically, this dataset includes $1,000$ classes with a relatively large training set ($1.2M$ images, $\smallsim\!\!10\times$ larger than datasets adopted in main paper). The validation set has $50,000$ images and $50$ for each class for retrieval. We randomly pick $5$ images of each class as queries ($5,000$ in total) and the remaining is adopted to formulate the base split ($45,000$ in total). Networks are trained for $20$ epoch and we only update the last hash layers since the backbone is already pre-trained on ImageNet. Other settings are the same with main paper. From \cref{Tab.LargeSet}, we could see our method is also effective when trained with large dataset.

\subsection{Quantitative Comparisons}

\paragraph{Precision-Recall Curves.} To indicate retrieval performance over the whole rank list, we plot precision-recall curves for all methods in the first row of \cref{Fig.AllCurves} for three datasets on $64$ bits. We obtain similar results with Tab.2 where all the methods obtain performance gain (curves move to upper-right) after integrating our methods.

\paragraph{Precision@$K$ Curves.} To visualize retrieval precision of top retrieved samples, Precision@$K$ ($K$ from $100$ to $1,000$) curves for three datasets on $64$ bits are placed in the second row of \cref{Fig.AllCurves}. Similar to P-R curves, most of the methods receive a precision increase and curves slightly move up.

\begin{wraptable}{R}{0.5\textwidth}
\caption{Retrieval performance on the full ImageNet dataset.}
\label{Tab.LargeSet}
\centering
\begin{tabular}{@{}rccc@{}}
\toprule
Method                & 16 bits & 32 bits & 64 bits \\ \midrule
HashNet                   & $36.9$ &$41.2$  & $44.7$  \\
HashNet\rlap{$_D$}        & $\bm{55.0}$\diffH{18.1}&$\bm{56.1}$\diffH{14.9} &$\bm{61.8}$\diffH{17.1}  \\ \midrule
DBDH                   & $37.4$ &$41.1$  & $49.1$  \\
DBDH\rlap{$_D$}        & $\bm{57.2}$\diffH{19.8}&$\bm{57.9}$\diffH{16.8} &$\bm{60.4}$\diffH{11.3}  \\  \midrule
DSDH                   & $39.3$ &$48.4$  & $54.2$  \\
DSDH\rlap{$_D$}        & $\bm{56.0}$\diffH{16.7}&$\bm{59.1}$\diffH{10.7} &$\bm{62.0}$\diffH{7.8}  \\  \midrule
DCH                   & $60.2$ &$62.3$  & $64.1$  \\
DCH\rlap{$_D$}        & $\bm{61.9}$\diffH{1.7}&$\bm{62.9}$\diffH{0.6} &$\bm{64.4}$\diffH{0.3}  \\  \midrule
GreedHash                   & $62.7$ &$63.1$  & $64.9$  \\
GreedHash\rlap{$_D$}        & $\bm{63.1}$\diffH{0.4}&$\bm{64.0}$\diffH{0.9} &$\bm{65.9}$\diffH{1.0}  \\  \midrule
CSQ                       & $64.6$ &$65.0$  & $65.7$    \\
CSQ\rlap{$_D$}            & $\bm{65.7}$\diffH{1.1}  & $\bm{66.2}$\diffH{1.2} & $\bm{66.4}$\diffH{0.7} \\ \bottomrule
\end{tabular}
\end{wraptable}

\paragraph{Precision@$H=2$ Curves.} To show retrieval performance inside $H=2$ Hamming ball, we draw Precision@$H=2$ \wrt code-length curves on the third row of \cref{Fig.AllCurves}. With our methods, precision inside Hamming ball is also increased. This indicates that our method pushes more true positives close to queries than the original methods.

\subsection{Ablation Study}
\subsubsection{Posterior Estimation on Longer Bits.}
Due to space limitation in main paper, we place explanation on how we estimate $\MVBer$ by na{\"i}ve way here (Sec.6.2). Then, an extra experiments on longer bits \ie $16$ bits will be conducted to validate the effectiveness of blocked estimation.

\begin{figure}[t]
    \centering
    \includegraphics[width=\linewidth]{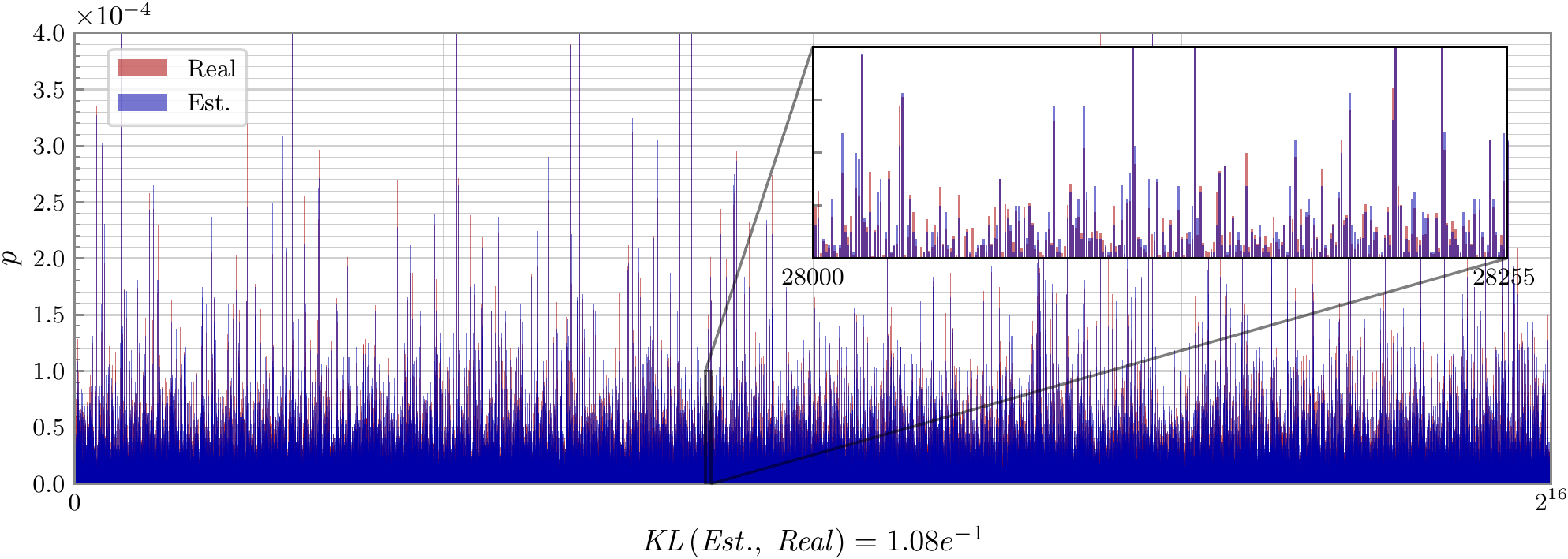}
    \caption{Posterior estimation on $16$ bits codes ($65,536$ joint probabilities) with a block=$2$ surrogate model. We randomly show a zoomed-in view with interval of $256$ on the upper-right. Ours $\mathit{KL}$: $0.1084$, na{\"i}ve's $\mathit{KL}$: $0.4965$.}
    \label{Fig.Sanity16}
\end{figure}

For na{\"i}ve estimation on $\MVBer$, dependence among $\bm{b}$ is ignored. Specifically, we directly count frequency of value to be greater than zero on each bit of codes to approximate $p\left(\bm{b}_i > 0\right)$ and $p\left(\bm{b}_i < 0\right)$. Then, any joint probability is estimated by the product of marginal probabilities, \eg $p\left(\bm{b} = \text{+1+1-1-1}\right) \estimates p\left(\bm{b}_4 > 0\right) \cdot p\left(\bm{b}_3 > 0\right) \cdot p\left(\bm{b}_2 < 0\right) \cdot p\left(\bm{b}_3 < 0\right)$.

To train on $16$ bits codes, we adopt a block=$2$ surrogate model where the first block is used to produce the first $8$ bits codes and the second one is for the last $8$ bits (as in \cref{Fig.Code} implements). Then, for all $65,536$ joint probabilities, we obtain them by the Cartesian product of two models' predictions, which is shown in \cref{Fig.Sanity16}. The figure shows similar results as the $8$ bits experiment in main paper. Our model achieves a much lower $\mathit{KL}$ than na{\"i}ve one, showing that performing block estimation will not introduce significant bias to estimate posterior (Ours: $0.1084$, na{\"i}ve: $0.4965$).

\begin{table*}[!h]
\caption{mAP comparisons with different pre-defined centers on three benchmark datasets for $16, 32, 48$ bits codes. Values on lower right is minimum pairwise Hamming distances among centers.}
\label{Tab.Center}
\centering
\resizebox{\linewidth}{!}{
\begin{tabular}{@{}rccccccccc@{}}
\toprule
\multirow{2}{*}{Center}&\multicolumn{3}{c}{\textbf{CIFAR-10}}&\multicolumn{3}{c}{\textbf{NUS-WIDE}}&\multicolumn{3}{c}{\textbf{ImageNet}}\\ \cmidrule(lr){2-4}\cmidrule(lr){5-7}\cmidrule(lr){8-10}
                       &  $16$ bits &  $32$ bits & $48$ bits &  $16$ bits &  $32$ bits & $48$ bits &  $16$ bits &  $32$ bits & $48$ bits \\ \midrule
Reed-Solomon           &  $87.7_{6}$    &  $89.0_{14}$   & $89.8_{21}$     &  $83.1_{5}$    &  $84.9_{12}$    & $85.0_{19}$    &  $88.5_{4}$    &  $89.5_{10}$    & $90.2_{16}$     \\
BCH                    & $88.3_{7}$ & $89.2_{16}$ & $89.9_{21}$ & $83.0_{4}$ & $84.8_{12}$ & $85.1_{19}$ & $87.3_{2}$ & $89.0_{9}$ & $89.8_{14}$ \\

Hadamard               &  $88.7_{8}$    &  $89.2_{16}$    & $89.5_{19}$    &  $83.3_{8}$    &  $85.3_{16}$    & $85.0_{19}$    & $87.9_{3}$ & $89.0_{8}$ & $89.8_{14}$     \\ \bottomrule
\end{tabular}}
\end{table*}

\subsubsection{Impact of Different Kinds of Pre-defined Centers.}
To generate pre-defined centers as separate as possible as targets to train model, we compare a series of error-correction codes. The experiment includes Reed-Solomon code, BCH code and Hadamard code. Four kinds of codes have different pairwise Hamming distance. According to our main proposition in paper, if the minimum pairwise Hamming distance between centers are smaller, the lower bound of hash codes' performance is tend to be correspondingly lower. To show this phenomenon, minimum pairwise Hamming distance of codes as well as mAP are measured for CSQ$_D$, which is placed in \cref{Tab.Center}. It is worth noting that Hadamard code is not applicable when number of centers is large or code-length is not order of $2$, in this case we report results with randomly generated codes. The table confirms the positive correlation between minimum pairwise Hamming distance and performance. Meanwhile, performance differences among three kinds of codes are small. Therefore in practice, we could choose any of them.

\end{document}